\documentclass[10pt]{article} 
\usepackage[preprint]{tmlr}

\usepackage{amsmath,amsfonts,bm}

\def\eqref#1{equation~\ref{#1}}

\def\1{\bm{1}}

\def\vz{{\bm{z}}}

\DeclareMathAlphabet{\mathsfit}{\encodingdefault}{\sfdefault}{m}{sl}
\SetMathAlphabet{\mathsfit}{bold}{\encodingdefault}{\sfdefault}{bx}{n}

\usepackage[utf8]{inputenc} 
\usepackage[T1]{fontenc}    
\usepackage{hyperref}       
\usepackage{url}            
\usepackage{booktabs}       
\usepackage{amsfonts}       
\usepackage{nicefrac}       
\usepackage{microtype}      
\usepackage{xcolor}         

\usepackage[pdftex]{graphicx}
\usepackage{microtype}
\usepackage{subcaption}
\usepackage{multirow}
\usepackage{multicol}
\usepackage{amsthm}
\hypersetup{
    colorlinks,
    linkcolor={red!50!black},
    citecolor={blue!50!black},
    urlcolor={blue!50!black}
}

\usepackage{wrapfig}
\usepackage{xcolor}
\usepackage[normalem]{ulem} 

\usepackage[capitalize,noabbrev]{cleveref}

\usepackage{amsmath}
\usepackage{amssymb}
\usepackage{mathtools}
\usepackage{amsthm}

\theoremstyle{plain}
\newtheorem{theorem}{Theorem}[section]

\theoremstyle{definition}

\newtheorem{assumption}[theorem]{Assumption}

\theoremstyle{remark}

\usepackage{thmtools}
\usepackage{thm-restate}

\usepackage[capitalize,noabbrev]{cleveref}

\usepackage{tcolorbox}
\tcbuselibrary{theorems}

\newtcbtheorem[number within=section]{assumptionbox}{Assumptions}%
{colback=gray!5,colframe=gray!40!black,fonttitle=\bfseries}{assumpbox}

\newcommand{\absl}[1]{\left| #1 \right|}

\title{When does Subgroup Balance Matter? Predicting Performance Sensitivity from Latent Separability}

\author{\name Anissa Alloula \email anissa.alloula@dtc.ox.ac.uk \\
      \addr
      University of Oxford
      \AND
      \name Charles Jones \\
      \addr Imperial College London
      \AND
      \name Zuzanna Wakefield-Skorniewska\\
      \addr University of Oxford
      \AND
      \name Francesco Quinzan\textsuperscript{*}\\
      \addr University of Oxford
      \AND
      \name Bart\l{}omiej W. Papie\.z\textsuperscript{*}\\
      \addr University of Oxford
}

\begin{document}
\maketitle
\begingroup
\renewcommand{\thefootnote}{*}
\footnotetext{Equal senior contributors.}
\endgroup
\begin{abstract}
Unequal representation of demographic groups in training data poses challenges to model generalisation across populations. Standard practice assumes that balancing subgroup representation always optimises performance. However, recent empirical results contradict this assumption: in some cases, imbalanced data distributions actually improve subgroup performance, and in others, subgroup performance is unaffected by the absence of an entire subgroup during training. We propose the \textit{latent separation hypothesis}, which states that a fine-tuned model's dependence on subgroup representation is determined by the degree of separation between subgroups in the latent space of the pre-trained model. We formalise this hypothesis, provide theoretical analysis, validate it empirically across four vision and language datasets, and present preliminary evidence of a causal role through regularisation experiments. Finally, we show how quantitative analysis of latent subgroup separation can inform practical data collection and balancing decisions when fine-tuning foundation models.
\end{abstract}


\section{Introduction}
\label{intro}
There is a wide consensus in machine learning that model performance improves monotonically with increasing training data \citep{rosenfeldCONSTRUCTIVEPREDICTIONGENERALIZATION2020,kaplanScalingLawsNeural2020}. This principle, formalised through dataset scaling laws, has guided much of the recent progress in model training. However, real-world data rarely satisfies the assumption of being independent and identically distributed (i.i.d.) \citep{arjovsky2020invariantriskminimization,wangFairnessenhancingMixedEffects2023}. Instead, datasets are composed of clusters of correlated samples, corresponding to subgroups or domains. In the medical domain, clusters may correspond to demographic categories, while in image datasets they may reflect camera types, and in multilingual corpora they may represent language varieties.

In such cases, the question becomes more nuanced: \textbf{how does model performance on a particular subgroup scale as its representation in the training data increases?} While intuition suggests that a higher proportion of subgroup-specific data should directly improve performance on that subgroup, recent studies have revealed surprising counterexamples \citep{rolfRepresentationMattersAssessing2021a,wengsexbased2023,cevorapopulationshift2024}, where increasing subgroup allocation had little or even no effect. This challenges the widely held view that dataset rebalancing is always a reliable solution to improving subgroup performance \citep{idrissiSimpleDataBalancing2022}.

Therefore, understanding the relationship between subgroup allocation and subgroup performance remains an important open question. When concerned about fairness, practitioners must decide whether to conduct certain interventions, like collecting balanced data across demographic groups, or resampling or augmenting their dataset, potentially at the cost of reduced overall performance \citep{rajiActionableAuditingInvestigating2019,idrissiSimpleDataBalancing2022}. In domain generalisation, practitioners must weigh whether to fine-tune on a small set of domain-specific data or on a large set of general data \citep{hulkundDataS^3DatasetSubset2025}. More broadly, given the cost of data collection and annotation, knowing when subgroup representation matters can guide whether to prioritise general high-quality data or group-specific data.

In this work, we aim to understand how the allocation of data across subgroups affects subgroup performance at inference time, given a fixed training budget. Through extensive experiments in vision and language tasks, we find that subgroup sensitivity to allocation varies dramatically across datasets, models, and attributes. We probe \textit{why} these differences arise and put forward the \textit{latent separation hypothesis}: the degree to which redistributing subgroup data improves subgroup performance is determined by how strongly the set of subgroups are separated in the model’s latent representations. We provide both theoretical justification and empirical evidence for this hypothesis. 

\begin{figure*}[t!]
    \centering
    \includegraphics[width=1\linewidth]{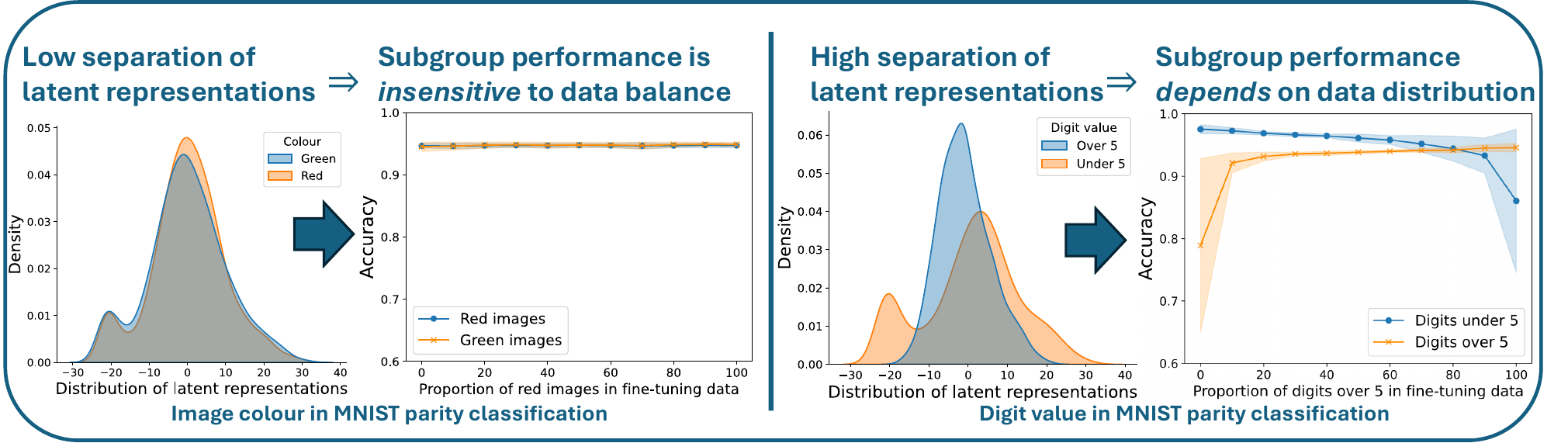}
    \caption{\textbf{Model sensitivity to data balance depends on latent separation of subgroups.} Left plots show PCA projections of latent representations of MNIST parity classifiers. Right plots show subgroup accuracy as training data allocation changes.}
    \label{fig:mnist_reps_allocs}
\end{figure*}

\paragraph{Our contributions are:}
\begin{itemize}
    \item We demonstrate that widely held explanations for sensitivity to subgroup allocation fail to match empirical behaviour (\S\ref{sec:current_hypotheses}).
    \item We derive a theoretical upper bound on sensitivity to subgroup allocation based on subgroup separation in the pre-trained model’s latent space (\S\ref{sec:theory}).
    \item We show empirically that sensitivity to subgroup allocation is correlated to the distance between two subgroups' latent representations, and give preliminary evidence this link is mechanistic~(\S\ref{sec:experiments}).
    \item We show how our findings can guide dataset selection decisions to improve fairness in a case-study fine-tuning two vision-language foundation models (\S\ref{sec:vlm_fine-tuning}).
\end{itemize}

\section{Related work}
\label{related_work}
\paragraph{Dataset scaling is not straightforward when the train and deployment settings are not i.i.d.}
We are broadly interested in the relationship between model training data and model performance. Research in this area has investigated different aspects of data (e.g., size, composition, or individual points) and different performance metrics (e.g., overall loss, fairness, or domain-specific accuracy) \citep{hashimotoModelPerformanceScaling2021}. While dataset scaling laws have shown that performance improvements follow predictable power law trends \citep{rosenfeldCONSTRUCTIVEPREDICTIONGENERALIZATION2020,kaplanScalingLawsNeural2020}, this relationship becomes more complex when the training and test data are drawn from different distributions. In such cases more data is not necessarily better. For instance, \cite{hulkundDataS^3DatasetSubset2025} and \cite{shenDataAdditionDilemma2024} both show that when optimising for a specific deployment setting, a subset of the data can yield better performance than the full dataset. Similarly, \cite{diazfairnessscalinglaws} argue that scaled training data can have a negative impact depending on the evaluation metrics and subpopulations considered.

\paragraph{Subgroup data scaling through the lens of fairness}
This problem has also been studied indirectly in the field of fairness, where data are grouped into subgroups (e.g., based on demographic attributes), and one investigates how training data composition (i.e., number/proportion of samples from certain subgroups) affects fairness (i.e., some metric based on model performance on certain subgroups). The prevailing assumption is again that more subgroup-specific training data leads to improved performance on that subgroup \citep{rajiActionableAuditingInvestigating2019,chenWhyMyClassifier2018}. When more data cannot be collected, the standard intervention is to rebalance the model training data by under- or over-sampling samples from certain subgroups \citep{idrissiSimpleDataBalancing2022}. Many works show that this simple method can yield remarkable fairness improvements both when training from scratch \citep{idrissiSimpleDataBalancing2022} and when fine-tuning (potentially biased) pre-trained models \citep{kirichenkoLASTLAYERRETRAINING2023,wangOverwritingPretrainedBias2023,alabdulmohsinCLIPBiasHow2024}. 

\paragraph{Inconsistencies in subgroup balancing results}However, a growing body of work argues that balancing data does not necessarily improve fairness, and that it can even be detrimental. \cite{schrouffMindGraphWhen2024} use a causal framework to show conditions under which data balancing will \textit{not} improve model fairness. Similarly, \cite{schwartz-stanovsky-2022-limitations}, \cite{rohFairBatchBatchSelection2021}, \cite{qiao2025grouprobust}, and \cite{claucichFairnessDeepEnsembles2025} show that fairness is not necessarily maximized at uniform group ratios, with the latter arguing that this is due to unequal task difficulty across groups. \cite{wengsexbased2023} and \cite{cevorapopulationshift2024} even show cases where a model's performance on female medical images remains constant (and sometimes decreases) as the proportion of female images in the training set increases. Loss-reweighting based approaches such as groupDRO \citep{sagawaDistributionallyRobustNeural2020} are based on a related principle: rather than balancing group size, they reweight groups with higher losses. However, loss-based methods also exhibit failure modes \citep{zongMEDFAIRBENCHMARKINGFAIRNESS2023}, and it remains unclear whether up-weighting data from poor-performing groups necessarily improves performance on those groups. Together, these studies reflect an emerging trend in fairness research, that different subgroups have distinct properties and causes of under-performance, and therefore respond differently to interventions like data balancing or loss re-weighting \citep{wangIdentitiesAreNot2025a, alloula2025subgroups, jonesRethinkingFairRepresentation2024,yangChangeHardCloser2023}. However, without a causal framework (which is difficult to apply in practice) \citep{jonesRethinkingFairRepresentation2024,schrouffMindGraphWhen2024} or direct experimentation, it is difficult to determine \textit{a priori} whether balancing will improve fairness.  

\paragraph{Impact of subgroup allocation on subgroup performance}
Our work differs from fairness-oriented studies in that we address a more fundamental question: how does subgroup allocation affect subgroup performance? We argue that understanding this is a prerequisite for tackling fairness concerns and implementing any bias mitigation methods (as many of these involve re-allocating subgroup data). Despite its importance, this question has received little direct attention, and as discussed above, it is not clear whether redistributing data from under-represented or poorly performing groups reliably improves performance for those groups. \cite{rolfRepresentationMattersAssessing2021a} take a first step by fitting a per-group power-law scaling model describing the impact of subgroup and total training data size on subgroup performance. Similarly to the fairness papers, they show that the optimal allocation varies across datasets and tasks, and is not necessarily balanced. Our work builds on theirs in several ways, but differs crucially in that we propose a (theoretically-grounded) explanation for \textit{why} subgroup allocation impacts subgroup performance so variably. Unlike \cite{rolfRepresentationMattersAssessing2021a}, who must train many models across dataset sizes and allocations to fit their empirical scaling law, we identify an underlying mechanism driving these effects. This enables us to determine, for a given model and subgroup, whether subgroup allocation is likely to matter, and can help guide fine-tuning strategies to maximise subgroup performance.


\section{Problem setting}
\label{background}

To study the impact of subgroup allocation on subgroup performance, we consider supervised fine-tuning of a pre-trained model on a dataset of input-label pairs 
\((x,y) \in \mathcal{X} \times \mathcal{Y}\). The data are drawn from an underlying 
distribution \(\mathcal{P}\), which we randomly split into three disjoint subsets: 
\(\mathcal{P}_{\text{pre-train}}\), \(\mathcal{P}_{\text{fine-tune}}\), and 
\(\mathcal{P}_{\text{test}}\). We study settings where the training and test distributions are annotated with \(m\) 
binary attributes \(\{A^{(1)}, \dots, A^{(m)}\}\) which can represent demographic or other sample-level characteristics. Each attribute \(A^{(j)}\) induces a binary partition of the data into two subgroups, \(a^{(j)}_0 = \{(x,y) \mid A^{(j)} = 0\}\) and \(a^{(j)}_1 = \{(x,y) \mid A^{(j)} = 1\}\). Examples of attributes include gender (male/female), imaging view (frontal/lateral), or dataset source (scanner A/scanner B).

For each subgroup, we record its base population prevalence under $\mathcal{P}$ :
$\gamma^{(j)}_k 
= \Pr_{(X,Y,A^{(j)}) \sim \mathcal{P}}[A^{(j)} = k], 
\quad k \in \{0,1\}.$ During fine-tuning, we investigate the impact of manipulating the prevalence of each subgroup, which we refer 
to as \textbf{subgroup allocation}. We assume there is a fixed fine-tuning budget of \(K\) examples \(\{(x_i, y_i, A_i^{(j)})\}_{i=1}^K\). For each subgroup \(a^{(j)}_k\), following \cite{rolfRepresentationMattersAssessing2021a}, we define its allocation as the fraction of the fine-tuning dataset coming from subgroup \(a^{(j)}_k\):
\[
\alpha^{(j)}_k := \frac{1}{K} \sum_{i=1}^{K} \mathbf{1}[A_i^{(j)} = k], \quad k \in \{0,1\}.
\]

Our objective is to characterise how subgroup-specific test performance (for instance the loss $\ell^{(j)}_k$) depends on these allocations. The central question of this work is thus:
how does $\ell^{(j)}_k$ vary with $\alpha^{(j)}_k$ and why does this sensitivity differ across attributes and subgroups?

\section{Current explanations are unreliable}
\label{sec:current_hypotheses}
We begin by systematically comparing how subgroup allocation affects subgroup performance across various empirical settings. We explore whether existing hypotheses, for instance that under-represented subgroups will benefit from increased allocation, can explain the patterns we observe.

\subsection{Experimental setup}

\paragraph{Model training with different allocations} 
We start with a baseline model trained on a random subset of the original dataset and subsequently fine-tune it on datasets for which we systematically vary subgroup allocations $\alpha$. For each attribute $A^{(j)}$, we partition the dataset into binary subgroups $g^{(j)}_0$ and $g^{(j)}_1$. For each attribute and dataset, we create 11 fine-tuning datasets, varying the allocation $\alpha^{(j)}_1 \in \{0,0.1,0.2,\dots,1\}$ while keeping the total fine-tuning dataset size $K$ constant (ablations on $K$ are presented in Figure~J\ref{fig:mimic tv corr acros sizes}). Correspondingly, $\alpha^{(j)}_0 = 1 - \alpha^{(j)}_1$. This yields, for instance, a dataset with 0\% female images, another with 10\% female images, and so on until 100\%. 


\paragraph{Assessing sensitivity to subgroup allocation}
We quantify how subgroup allocation affects model performance on each subgroup $g^j$ in two ways. First, we fit a linear least-squares regression to the subgroup-specific performance (e.g., accuracy or loss) as a function of the allocation $\alpha^{(j)}_k$, recording the slope $a^{(j)}_k$. We then average this across both groups to obtain a slope estimate $a^j$. We also obtain a coarse estimate of generalisation by subtracting model performance on subgroup $g^{(j)}_k$ at 0\% allocation from its performance at 100\% allocation:
$\Delta \ell^{(j)}_k = \ell^{(j)}_k(\alpha^{(j)}_k = 1) - \ell^{(j)}_k(\alpha^{(j)}_k = 0).$ A model generalises very well on female data if, e.g., its loss is 0.1 whether or not it has been trained on female samples.

\paragraph{Datasets, tasks, and models}

We conduct these experiments in four image and text datasets, using CNNs and transformers for binary classification tasks. This includes even/odd digit prediction with a red and green coloured \textbf{MNIST} \citep{lecunGradientbasedLearningApplied1998}, pleural effusion classification in \textbf{MIMIC-CXR} (chest-X rays) \citep{johnsonMIMICCXRJPGLargePublicly2019}, skin lesion detection in \textbf{HAM10000} (skin images) \citep{ham10000dataset}, and toxic comment classification in \textbf{Civil\_comments} \citep{civilcommentsdataset}. These datasets all contain various metadata which enables natural splitting of the samples into subgroups, based on attributes like sex, ethnicity, image type, date etc. We compare many attributes within the same dataset, which allows us to gain more insights than previous studies which only consider one or two standard groupings per dataset \citep{rolfRepresentationMattersAssessing2021a, claucichFairnessDeepEnsembles2025,idrissiSimpleDataBalancing2022}. Data characteristics and implementation details are provided in Tables \ref{sec: appendix-supplementary-details}\ref{tab: implementation details} and \ref{tab:subgroup-attributes}, and code is shared \href{https://anonymous.4open.science/r/subgroup-balance-latent-separation-FFE6/README.md}{here}.

\subsection{Sensitivity to subgroup allocation is highly variable}
\begin{figure*}[htbp!]
    \centering
    \includegraphics[width=1\linewidth]{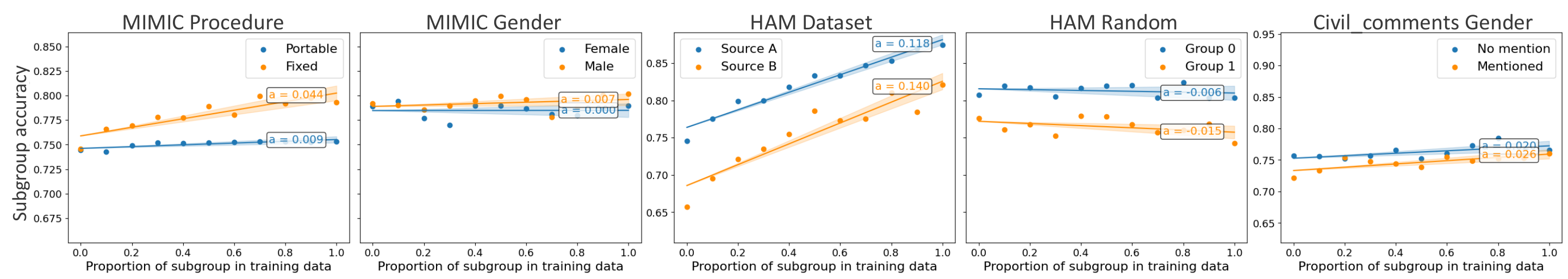}%
    \caption{\textbf{While some subgroups' accuracy increases with increased representation in training data, others' performance is independent of their training data representation}. The fine-tuned model's mean balanced accuracy on each subgroup across 9 runs and regression slopes $a$ are shown.}
\label{fig:subgroup alloc examples}
\end{figure*}
Across our three real-world datasets, we find substantial variability in subgroup performance sensitivity to allocation.  Certain subgroups show no benefit from increased allocation and achieve equivalent performance whether the model is fine-tuned only on that subgroup or entirely without it (e.g., age in MIMIC). In contrast, other subgroups, like dataset of origin in HAM10000, are sensitive to their allocation. This results in estimated slopes of the accuracy change which range from 0 (no effect) to 0.14 (strong sensitivity), as shown in Figure~\ref{fig:subgroup alloc examples}. In other words, while the MIMIC model performs equivalently on old individuals whether or not it has been trained on such images, the HAM model is almost 10\% less accurate on one dataset source if it has not been trained on any images from it. We summarise all slopes obtained in Tables C\ref{table:fullftsubgroupslopes} and C\ref{table:llrtsubgroupslopes}. These results mirror recent work which showed subgroup performance is not necessarily maximised when a dataset is balanced  \cite{rolfRepresentationMattersAssessing2021a,claucichFairnessDeepEnsembles2025,rohFairBatchBatchSelection2021,cevorapopulationshift2024,wengsexbased2023}. 

We consider more complex functional forms for fitting subgroup loss vs. allocation (e.g., power-law models \citep{rolfRepresentationMattersAssessing2021a}), but we find them unstable: fits vary substantially with small changes in data and standard deviations are large (further discussed in \S\ref{sec:complex-fits}), an issue also reported in \cite{rolfRepresentationMattersAssessing2021a}. We attribute this to small sample sizes and high heterogeneity in subgroup behaviour. Instead, we adopt linear fits as robust and interpretable first-order measure of sensitivity. The estimated slopes reliably capture the magnitude of the effect across allocations, and allow sensitivity to subgroup allocation to be compared across subgroups for any given model and dataset. To ensure possible threshold effects are taken into account, we also fit a linear model specifically at extreme allocations (when $\alpha \in [0,0.1])$) and elaborate in \ref{sec:allocation at the extremes}.

\subsection{Existing hypotheses fail to explain differences in allocation sensitivity}
We explore common explanations for sensitivity to subgroup allocation including whether it could be linked to certain subgroups being under-represented in the initial pre-training dataset, certain subgroups being disadvantaged by the pre-trained model's performance, or certain subgroups having substantial class imbalances. To test these three hypotheses, we investigate whether sensitivity to subgroup allocation (as measured by how much subgroup performance, or loss, changes as subgroup allocation increases) is correlated with: the prevalence of each subgroup in pre-training data (Figure~\ref{fig:correlation subgroup prevalence accuracy slope} left), how well the base model performs on each subgroup (middle), and to whether there is class imbalance across subgroups (right). Figure D\ref{fig: full alt hypotheses no normalisation} and Table D\ref{appendix-tab:correlation_alt_hypotheses} also show the full results and statistical significance values for each dataset. 

None of these explanations are consistently correlated with sensitivity to subgroup allocation. For instance, we see certain subgroups, like patients with support devices in MIMIC, which are extremely under-represented in the pre-training set (e.g., less than 20\% of the pre-training data) which show no reduction in loss as fine-tuning data allocation increases (Figure~\ref{fig:correlation subgroup prevalence accuracy slope} left). This suggests that \textbf{over-representing under-represented subgroups does not necessarily yield performance improvements}, in line with other recent work \citep{schrouffDiagnosingFailuresFairness2023,claucichFairnessDeepEnsembles2025,rohFairBatchBatchSelection2021} and contradicting other research \citep{idrissiSimpleDataBalancing2022,wangOverwritingPretrainedBias2023,alabdulmohsinCLIPBiasHow2024}. Similarly,  we see surprising examples where certain subgroups which are initially amongst the lowest performing by the pre-trained model also show almost no improvement as allocation increases (Figure~\ref{fig:correlation subgroup prevalence accuracy slope}, middle), again contradicting the general assumption that training on more data from a poor-performing subgroup will improve performance on that subgroup (or improve it more than training on general data) \citep{rohFairBatchBatchSelection2021,sagawaDistributionallyRobustNeural2020}. This implies that \textbf{model performance on \textit{difficult} subgroups} (whether due to noise, more difficult presentation of the task etc.) \textbf{will not necessarily improve with increased training data allocation}. Thirdly, subgroups with high class imbalance also do not necessarily benefit from increased subgroup representation, even if this changes the class distribution (Figure~\ref{fig:correlation subgroup prevalence accuracy slope}, right). Importantly, these three hypotheses underpin many existing bias mitigation methods, suggesting they may not be appropriate for use in all scenarios.

\begin{figure}[htb!]
    \centering
    \includegraphics[width=1\linewidth]{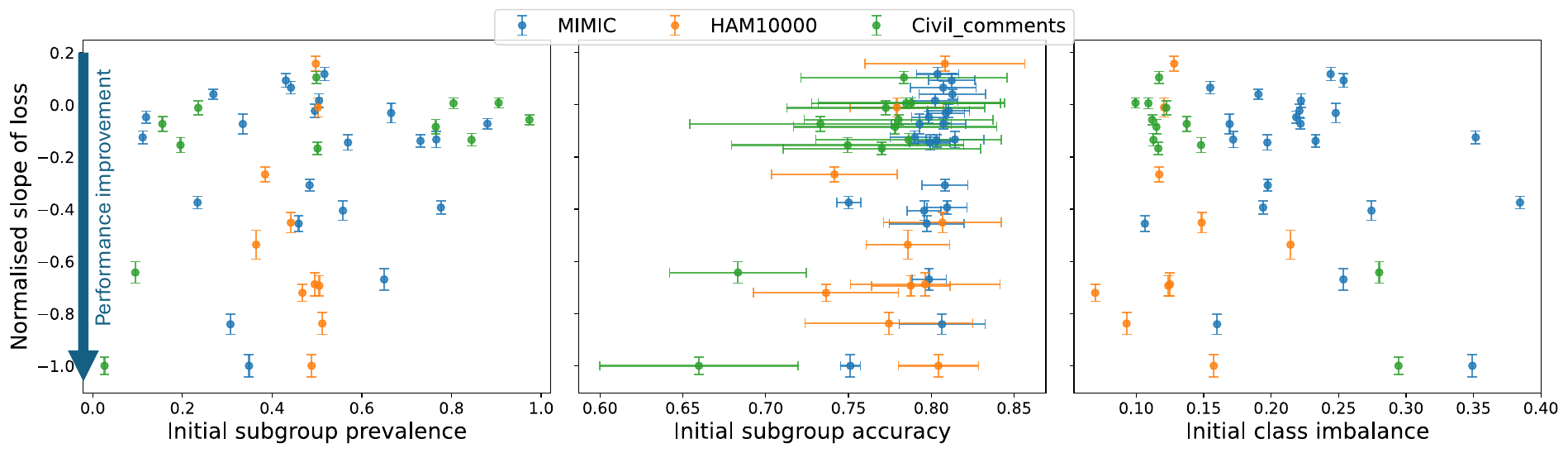}%
    \caption{\textbf{Performance on subgroups under-represented} during pre-training (left), \textbf{disadvantaged} (middle), and with \textbf{class imbalance} (right) \textbf{does not necessarily improve with increasing dataset allocation}. No clear correlation to performance change (a negative gradient of subgroup loss change indicates performance improvement) is observed in any of the 3 settings. Each point represents one subgroup, with error bars showing standard deviation across 9 runs.}
    \label{fig:correlation subgroup prevalence accuracy slope}
\end{figure}

\section{Bounding sensitivity to subgroup allocation with latent representations}
\label{sec:theory}
Given the lack of a coherent explanation for differences in sensitivity to subgroup balance, we introduce the \textit{latent separation hypothesis}: sensitivity to subgroup allocation depends on whether a model learns distinct representations for subgroups $a_0$ and $a_1$, and thus needs to be trained on a sufficiently high proportion of samples from, say $a_0$, in order to achieve good performance on $a_0$. This hypothesis is inspired by work in the field of domain adaptation, which relates target-domain generalisation error to the divergence between source and target distributions in the learned representation space \citep{shaiNIPS2006_b1b0432c}.

Through theoretical analysis and under certain assumptions, we validate this hypothesis. We show that low latent separation with respect to an attribute $A$ implies low sensitivity to $A$ allocation in the fine-tuning dataset. We formalise this idea by proving that small total variation distance (TV) between class-conditioned subgroup representations $Z$ in a pre-trained model, implies that last-layer fine-tuning on any dataset which differs in the allocation of $A$, cannot result in models which differ significantly specifically in their subgroup accuracies (Theorem \ref{theorem 1}). While this theorem is not a proof of our full hypothesis but rather an instance of it, it is an important sufficiency result, which, to the best of our knowledge, is the first to link subgroup allocation sensitivity to class-conditional representation separation. Full derivations are shown in \S\ref{appendix:full-theory}.
\begin{restatable}{assumption}{MainAssumptin}
\label{assumption}
We assume the following setup. 
We consider binary classifiers of the form $f_{\theta}(x)=g_\theta(h(x))$, where a fixed pre-trained representation $h$ maps inputs to $Z=h(X)$ and only the last-layer classifier $g_\theta$ is fine-tuned.
We compare two such classifiers, $g_{\theta'}$ and $g_{\theta''}$, trained on datasets $\mathtt D'$ and $\mathtt D''$ that differ only in their subgroup allocations $\alpha^{(A)}$.
All evaluations are performed on a fixed test distribution $\mathbb{Q}$ over $(X,Y,A)$.
\end{restatable}
\paragraph{Notation.} Let $\mathsf{TV}(\mu, \nu) = \sup_{S} |\mu(S) - \nu(S)|$ denote the total variation distance between two probability measures $\mu$ and $\nu$. For a classifier $f_\theta$ and a test distribution $\mathbb{Q}$, let $\textsc{Acc}_\theta$ denote its overall accuracy, and $\textsc{Acc}_\theta(A=a)$ denote its accuracy conditional on the sensitive attribute $A=a$.

\begin{restatable}{lemma}{MainLemma}
\label{lemma:tv}
Let $f_{\theta, \eta}(x)=g_\theta(h_\eta(x))$ with representation $Z=h_\eta(X)$, where $g_\theta$ is the last layer. Let $\mathsf{TV}(A,B)$ represent the total variation distance between two distributions $A$ and $B$.
Assume that
\begin{equation*}
\mathsf{TV}(\mathbb{Q})
:= \mathsf{TV}\!\left(
\mathbb{Q}[Z\!\mid\!Y{=}y,A{=}1],
\mathbb{Q}[Z\!\mid\!Y{=}y,A{=}0]
\right)
\le \varepsilon
\end{equation*}
for $y\in\{0, 1\}$. Then, it holds
$\bigl|\mathbb{Q}(Z\in S\mid Y=y,A=a)-\mathbb{Q}(Z\in S\mid Y=y)\bigr|
\le\varepsilon$ for all possible events $S$, for all attributes $a$.
\end{restatable}
This lemma tells us that if the representation \( Z \) is similar between groups (i.e., \( A = 0 \) and \( A = 1 \)) for a given label \( Y = y \), then each group's representation is also close to the overall representation for that label, meaning group membership does not significantly affect the representation once the label is fixed. We can use this lemma to prove the main result.
\begin{restatable}{theorem}{MainThm}
\label{theorem 1}
Following the notation introduced in Lemma \ref{lemma:tv} and Assumption \ref{assumption}, consider two models $f_{\theta'}(x)=g_{\theta'}(h(x))$ and $f_{\theta''}(x)=g_{\theta''}(h(x))$, where the last layers $g_{\theta'}$,$g_{\theta''}$ are fine-tuned on $\mathtt D'$, $\mathtt D''$ respectively. Suppose that 
\begin{equation*}
\mathsf{TV}(\mathbb Q) \leq \varepsilon,
\end{equation*}
where $\mathbb Q$ is the distribution of the test set. Then, for all attributes $a$, it holds
$$
\absl{{\textsc{Acc}}_{\theta'}(A=a) - {\textsc{Acc}}_{\theta''}(A=a) } \leq 2\varepsilon + \absl{{\textsc{Acc}}_{\theta'} - {\textsc{Acc}}_{\theta''} }.
$$
\end{restatable}
This theorem tells us that if two models are fine-tuned on datasets that differ only in the group proportions (i.e., the distribution of \( A \)), and both models learn approximately group-invariant representations (i.e., \( \mathsf{TV}(\mathbb Q) \leq \varepsilon \)), then the accuracy on any group \( A = a \) will only differ by a term involving $\varepsilon$ and the overall accuracy differences  between the two models. The proof is shown in \ref{theoremproof}.


\paragraph{Remark (Scope of the last-layer assumption).}
Theorem~\ref{theorem 1} does not rely on retraining only the last layer. For any fixed representation map $h\!: \mathcal{X}\!\to\!\mathcal{Z}$, if the class-conditional subgroup distributions of $Z$ have $\mathrm{TV}$ bounded by $\varepsilon$ then the same accuracy bound holds for any readout $g_\theta$ trained on $Z$. In practice, however, evaluating TV in early-layer spaces can yield large values and a vacuous bound.

\paragraph{Remark (Tightness of bound).} In practice, for a fixed fine-tuning budget $K$, the overall accuracy difference is often negligible relative to subgroup accuracy differences, i.e.\ $|\Delta {\textsc{Acc}_
\theta}| \ll |\Delta \textsc{Acc}_\theta(A=a)|$ (Figures~E\ref{fig:overall acc sd plots}, E\ref{fig:upper bound}). TV is thus the dominant driver in the upper bound.



\section{Subgroup separation predicts sensitivity to allocation}
\label{sec:experiments}
We showed that if subgroup representations are nearly indistinguishable (as measured by TV), then modifying fine-tuning dataset subgroup allocation has little effect on downstream accuracies. We now turn to empirical analyses to test how well this theoretical upper bound captures real empirical behaviour, and to investigate whether finer-grained patterns, such as correlations between representation separation and allocation sensitivity, emerge beyond the bound's guarantee.

\subsection{Assessment of representation invariance}

We keep the same setup as previously, where models are pre-trained on a random subset of each of the four datasets, and their last layer is then fine-tuned with varying subgroup allocations. We also extend the setup to full fine-tuning and training from scratch to test whether our findings hold beyond the theory's assumptions. We quantify subgroup separation by extracting penultimate-layer embeddings and computing the mean \textbf{total variation} distance (TV) between subgroup distributions conditioned on $Y$. TV is bounded in $[0,1]$, with higher values indicating stronger separation. However it also has some limitations when estimated empirically. In particular, estimating the conditional distribution $Z\mid (A,Y) $ can be noisy, especially for small sample sizes where the estimate is known to be biased upwards (further discussed in \S \ref{appendix:tv-limitations}). For example, for HAM10000, where the smallest group is only 60 samples, TV estimates may be less reliable. Therefore, for completeness, we explore additional distance metrics including the \textbf{Wasserstein distance} (WD) and the \textbf{Fréchet distance} (FD). We give methodological details and show that all three distances metrics are correlated in \S\ref{sec: appendix-representation-invariance}.

We find that there is wide variation in subgroup separation in penultimate-layer representations (Figure F\ref{fig:TV differences all datasets}). Generally, image-related attributes (e.g., X-ray view in MIMIC or dataset in HAM10000) induce greater separation than demographic attributes. Some subgroups, such as Civil\_comments year (pre- vs. post-2017) or HAM10000 location (extremity vs. trunk), show separations comparable to random splits. These groups most likely do not have distinct disease or toxicity ($Y$) features.

\subsection{Intuition for our results in MNIST}
In our synthetic MNIST set-up, we \textit{know} that a good parity classification model should not rely on colour ($A^{(0)}$) to make a prediction, so learnt representations $\vz $ should be invariant to $A^{(0)}$, in other words $P(\vz \mid Y, a^{(0)}_0) = P(\vz \mid Y, a^{(0)}_1)$. In contrast, a model must encode some notion of which digit is represented in order to classify it into even vs. odd, and therefore its learned representation \textit{should} depend on whether the digit is over 5 or under 5 ($A^{(1)}$). We test this by training a 2-layer CNN. As expected, the penultimate layer model embeddings do cluster by digit $A^{(1)}$ but not by colour $A^{(0)}$ (Figure~\ref{fig:mnist_reps_allocs}). Quantitatively, the average TV between the images representing digits over 5 and those under 5 is 0.17, approximately twice the average distance between the red and green images (which is itself close to that between random groupings), as shown in Figure~F\ref{fig:TV differences all datasets}. 

Our hypothesis predicts that fine-tuning this model on datasets with different proportions of the same subgroups will show that the model is only sensitive to the allocation of the under 5/over 5 groups but not to the red/green groups. Indeed, we find that the subgroup accuracy on under 5 and over 5 images drops sharply as their fine-tuning dataset allocation drops, while the subgroup accuracy in the red/green images is roughly independent of the fine-tuning dataset allocation (Figure~\ref{fig:mnist_reps_allocs}, bottom row). The model can effectively generalise ``zero-shot'' to new colours, but not to unseen digit groups. Importantly, thanks to the controlled setup, both red/green and over 5/under 5 groups have equal $P(Y)$, equal base accuracy, and equal base allocation in the pre-training dataset, so none of those explanations would have predicted the observed discrepancy in subgroup allocation sensitivity.

\subsection{Subgroup separation is highly correlated to allocation sensitivity}
\begin{figure*}[t!]
    \centering
    \includegraphics[width=1\linewidth]{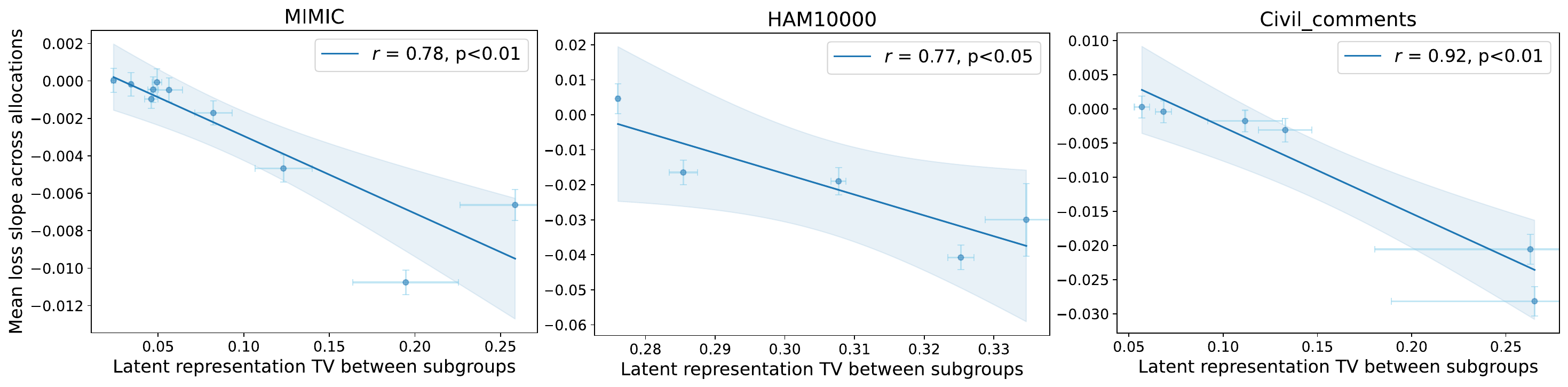}
    \caption{\textbf{Sensitivity to subgroup allocation is highly correlated with separation} in the pre-trained model's representation space (as measured by total variation distance, TV) across the three datasets. Each dot represents mean TV and loss slope for one subgroup, averaged across 9 fine-tuning runs, with bars corresponding to standard deviations, and Pearson correlation also shown.}
    \label{fig:subgroup alloc representations 3 datasets}
\end{figure*}
We next explore how our hypothesis holds in real-world datasets with more subgroups and larger models, where there is no clear-cut separation between attributes a model should or should not be invariant to. Across our three datasets and architectures, we find a significant \textbf{correlation between subgroup separation in pre-trained representations and subgroup sensitivity to allocation during fine-tuning} (Figure~\ref{fig:subgroup alloc representations 3 datasets}). This correlation holds across the three distance metrics, TV, WD, and FD, and when only extreme allocations are considered (Figures~C\ref{fig:correlation extreme allocs}, G\ref{fig:full corr distances}). Intuitively, when two subgroups have similar representation of target Y (i.e., $P(\vz \mid Y, a^{(0)}_0) = P(\vz \mid Y, a^{(0)}_1)$), then additional subgroup specific data provides little added value, and vice versa when there is separation.

We also observe that zero-shot generalisation is directly related to subgroup separation. As shown in Figure~F\ref{fig:llrt-zero-shot-generalisation-represenations-3-datasets}, subgroup accuracy is constant between 100\% and 0\% training data allocation when representations are (approximately) invariant across them. This echoes work in the domain adaptation literature, which stipulate that invariant representations are more robust as they generalise across environments \citep{shaiNIPS2006_b1b0432c,shaitheoryoflearning,arjovsky2020invariantriskminimization}. Similar ideas have also laid the foundation for ``fair representation learning'' methods which aim to mitigate biases by encouraging models to learn causal representations of $Y$ rather than encoding, e.g., demographic attributes like sex \citep{sarhanFairnessLearningOrthogonal2020,madrasLearningAdversariallyFair2018}. Here, we extend this principle to subgroup balancing by\textbf{ using a measurable
property of the \emph{pre-trained} model, class-conditional latent separation, as a
predictor for when allocation will matter}. Altogether, these empirical results support our theoretical finding of an upper bound to performance differences across allocations, and the consistent correlation suggests there may even be a stronger phenomenon at play.
\begin{figure*}[htbp!]
    \centering
    \includegraphics[width=\linewidth]{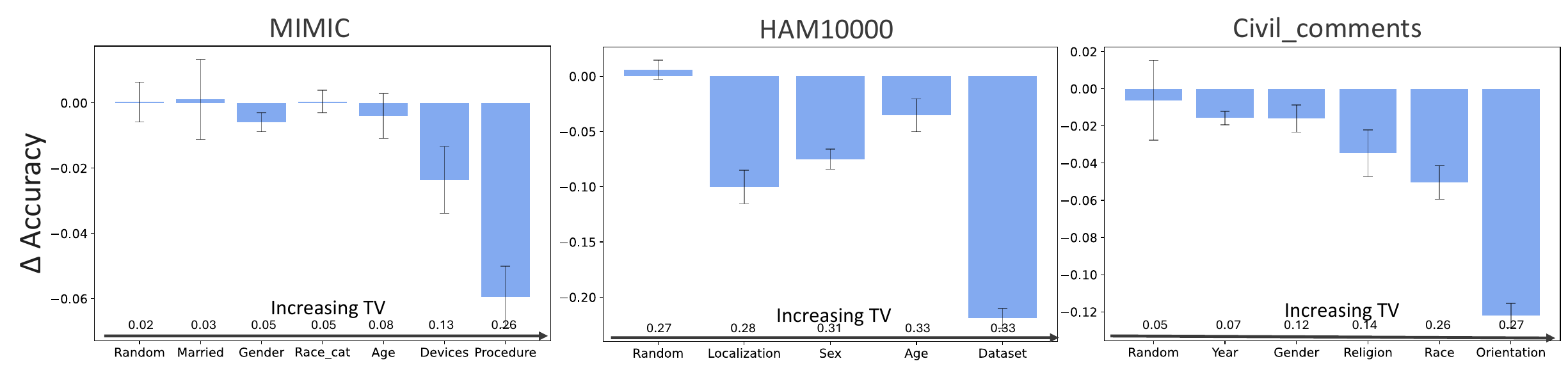}%
    \caption{\textbf{Subgroup generalisation degrades as latent subgroup separation increases.} Mean and standard deviation of balanced accuracy difference when a subgroup is excluded from training are shown across 3 training-from-scratch runs.}
    \label{fig:zero shot generalisation represenations 3 datasets}
\end{figure*}
\paragraph{The correlation also holds when full fine-tuning and training from scratch}
While we first experiment with last-layer fine-tuning for consistency with our theory, we further test whether this correlation extends to full fine-tuning and training from scratch. This departs from our initial settings as representations $Z$, and therefore latent separation, are no longer fixed. Despite this, we find equally strong correlations ($r \in [0.33,\,0.93]$), with effects of stronger magnitudes (loss slopes are up to 10x steeper than in last-layer fine-tuning), across all datasets and both training setups (Figures H\ref{fig:correlation slope distances full fine-tuning}, H\ref{fig:correlation slope distances scratch}). For example, Figure~\ref{fig:zero shot generalisation represenations 3 datasets} shows that performance drops on unseen subgroups are clearly higher for well-separated groups. Surprisingly, analysis of last-layer representations reveals that \textbf{latent separation is roughly constant across allocations, effectively matching our initial setup} (Figure H\ref{fig:reps at varying allocs}). This explains why we see the same trends, but stronger (as network parameters can vary more). We attribute this stability to the non-spurious nature of these subgroup shifts: if prediction features inherently differ across groups, separation should be high regardless of the training setup.

\subsection{Reducing latent separation decreases sensitivity to subgroup balance}
\begin{wrapfigure}{l}{0.38\linewidth} 
    \centering 
    \vspace{-1\baselineskip} \includegraphics[width=\linewidth]{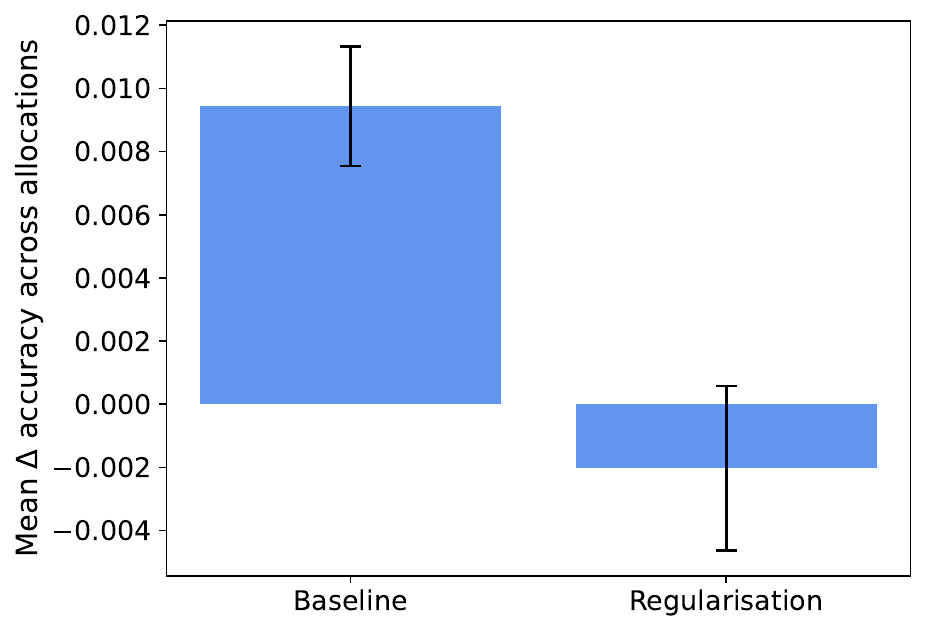} \caption{\textbf{TV regularisation reduces subgroup allocation sensitivity}; subgroup accuracy is constant.} \label{fig:operationalisation} \vspace{-10\baselineskip}
    \end{wrapfigure} 
We further test our hypothesis by explicitly \textit{enforcing} low TV during pre-training (via a differentiable proxy in the loss) to see whether this affects sensitivity to subgroup allocation (details in \S\ref{sec: appendix-operationalisation}). We find that \textbf{regularisation achieves its goal of reducing latent TV separation} (by approximately 13\%), and that \textbf{crucially, sensitivity to subgroup allocation is also reduced} (accuracy slope decreases from 0.01 to 0.00), as shown in Figure~\ref{fig:operationalisation}. This intervention directly supports our hypothesis that latent separation is not only predictive of allocation sensitivity, but that it can \textit{determine} allocation sensitivity.

\subsection{Application in foundation model fine-tuning}
\label{sec:vlm_fine-tuning}
\begin{wrapfigure}{r}{0.55\linewidth} 
\centering
\includegraphics[width=\linewidth]{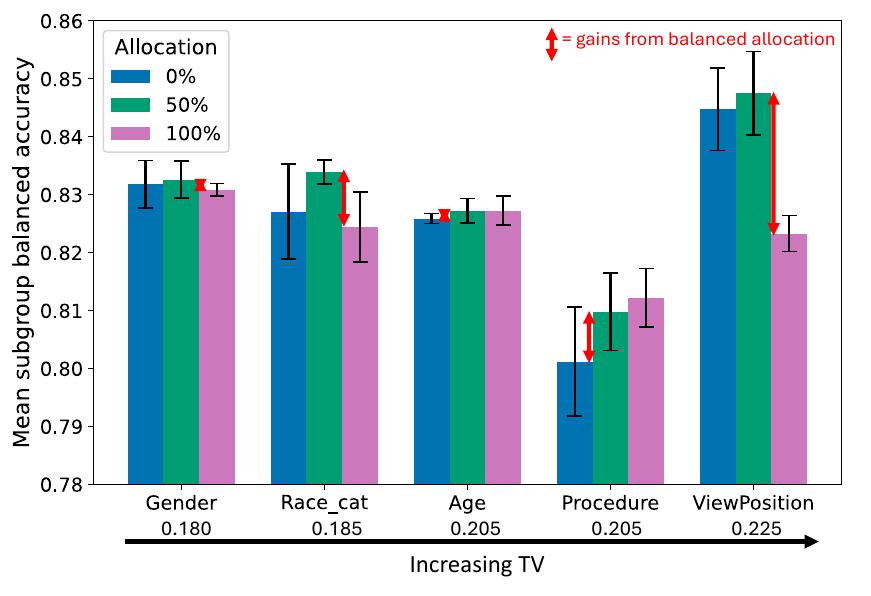}
\caption{\textbf{In foundation model fine-tuning}, balanced allocations increase subgroup accuracy by 0.02 for imaging groups, but yield little improvement for demographic groups, as predicted by their reduced TV.}
\label{fig vlm subgroup gains} 
\end{wrapfigure}
Beyond our findings' analytical value in explaining previously observed discrepancies in subgroup allocation sensitivity, we also examine its practical utility. Specifically, we consider a realistic setup where a foundation model (FM) is fine-tuned. We use two radiology FMs: CheXagent \citep{chen2024visionlanguagefoundationmodelenhance} and RAD-DINO-MAIRA-2 \citep{perezgarcia2024raddino}. Both are trained on over 1 million images, with distinct training mechanisms (image and text via SigLIP and only images via DINO respectively). We conduct last-layer fine-tuning and parameter-efficient fine-tuning. This significantly extends our traditional setup as the pre-training and fine-tuning datasets are drawn from different distributions and the pre-training and fine-tuning tasks do not directly overlap.

We again measure distances between subgroup representations and find substantial variability (Figure I\ref{fig:TV differences vlm}). Demographic attributes such as gender and age are more separated than in our task-specific model, likely because a general-purpose model must encode broad features. Across both models, the most separated groups remain imaging-related, followed by age, ethnicity, and gender. Based on our hypothesis, we predict that sensitivity to subgroup allocation should follow this ordering, and that if we are concerned about maximising subgroup performance across each of these attributes, we should prioritise balancing the dataset with respect to imaging variables. Our experiments confirm this: balancing by gender or race has little effect, but balancing by procedure or view noticeably increases mean subgroup accuracy (by about 3 percentage points) in both models (Figures~\ref{fig vlm subgroup gains}, I\ref{fig: maira-2 acc gains}, I\ref{fig:chexagent peft gains}). Full analysis of allocation sensitivity and latent separation shows strong correlations ($r \in [0.60,\,0.95]$, Figure~I\ref{fig:vlm distances slope}, \ref{fig:maira-2 distances slope}, \ref{fig:chexagent peft loss slope}), supporting our hypothesis under much weaker assumptions. Concretely, practitioners could \textbf{use this TV calculation to prioritise obtaining balanced data with respect to high-separation attributes} like imaging procedure and view position in further fine-tuning data collection efforts, and \textbf{otherwise spend budget on overall data quality and quantity}.

\subsection{Limitations and future work}

The core of our study is restricted to a pre-train–then–fine-tune setup. While this has the benefit of being an actionable setup, as one can first explore representation separation before deciding how to allocate data, it also means that the effects we observe are often modest in magnitude. However, they remain consistent and statistically significant across random seeds, data splits, and training setups. We further give strong empirical evidence that our hypothesis extends to training from scratch. We also do not provide a method for determining the optimal allocation strategy to maximise accuracy across multiple subgroups. We caution that precise recommendations (e.g., targeting 35/65 female/male ratios) would likely be unreliable due to factors like confounders, noise, or differences in sample informativeness, and in any case difficult to operationalise. Instead, we argue that a broader understanding of \textit{whether} and \textit{how strongly} subgroup allocation influences performance is more relevant to current concerns in fairness and generalisation. Moreover, future work should explore how our findings apply to more complex settings, including tasks other than binary classification, with multi-valued or even continuous subgroups, and settings where the training data is expanded instead of simply re-allocated. Further research should also investigate how our framework can be applied to data mixture design in fine-tuning general purpose LLMs, such as to help inform data balancing across languages or across application domains (e.g., coding, math, writing etc.). Finally, our work is limited, as in all fairness research, by our use of subgroups. For example, finding that a model is insensitive to the allocation of white vs. non-white samples does not imply that allocation across ethnic groups as a broad concept is irrelevant, simply that, when described by our chosen coarse groupings, it does not matter.


\label{results_and_discussion}

\subsection{Conclusions}
\label{conclusions}
We propose, prove (under certain assumptions), and give extensive empirical evidence for, a novel hypothesis explaining why in some cases subgroup training data allocation does not matter for subgroup performance, and in other cases, it is crucial. Unlike standard explanations (e.g., assuming that under-represented or poorly performing-groups always benefit from increased data representation), our hypothesis consistently matches empirical results across diverse datasets and models. By predicting subgroup allocation sensitivity through latent representation analysis, we provide a new way to inform crucial training data decisions, to maximise fairness, accuracy, and efficiency. 




\subsubsection*{Broader Impact Statement}
This paper analyses when subgroup balancing (e.g., collecting or upweighting data from specific subpopulations) is likely to improve subgroup performance. The main positive impact is making such interventions more targeted and effective while avoiding unnecessary data-related costs. Risks include misuse to justify avoiding investment in data for marginalised groups, and error from biased or noisy subgroup label and small-sample separability estimates. We recommend reporting uncertainty, validating conclusions across various experiments, and speaking to domain experts to check they agree with the interpretation of the hypothesis. Sensitive attributes should also always be handled with appropriate consent, privacy protections, and oversight.

\subsubsection*{Acknowledgments}
A.A. is supported by the EPSRC grant number EP/S024093/1, the Centre for Doctoral Training
SABS: R3, University of Oxford, and by GE Healthcare. C.J. is supported by Microsoft
Research, EPSRC, and The Alan Turing Institute through a Microsoft PhD scholarship and a
Turing PhD enrichment award. The computational aspects of this research were supported
by the Wellcome Trust Core Award Grant Number 203141/Z/16/Z and the NIHR Oxford
BRC. The views expressed are those of the author(s) and not necessarily those of the NHS,
the NIHR or the Department of Health.

\bibliography{icml2026_conference}
\bibliographystyle{tmlr}

\clearpage
\appendix
\section{Appendix structure}
\begin{itemize}
    \item \S\ \ref{sec: appendix-supplementary-details}: Supplementary details on datasets, models, and their implementations
    \item \S\ \ref{sec: appendix-subgroup-sensitivity}: Supplementary results on sensitivity to subgroup allocation and ways to model it
    \item \S\ \ref{sec: appendix-common-hypotheses}: Supplementary results on common hypotheses not holding
    \item \S\ \ref{appendix:full-theory}: Full theoretical results and further remarks
    \item \S\ \ref{sec: appendix-representation-invariance}: Supplementary results on latent representation distances
    \item \S\ \ref{sec: appendix-correlation-distance-sensitivity}: Supplementary results on the correlation between representation distance and subgroup allocation sensitivity
    \item \S\ \ref{sec: appendix-training-regimes}: Extending setup to other training regimes (full fine-tuning and training from scratch) and representation stability
    \item \S\ \ref{sec: appendix-vlm-finetuning}: Supplementary results on foundation model fine-tuning (PEFT and last-layer fine-tuning)
    \item \S\ \ref{sec: appendix-operationalisation}: Operationalisation of hypothesis through TV regularisation
    \item \S\ \ref{sec: appendix-finetuning-budget}: Ablations on fine-tuning budget $K$ in MIMIC
\end{itemize}

\section{Supplementary experimental details}
\label{sec: appendix-supplementary-details}

We present additional details on the four datasets and models used. We do not use the full datasets for fine-tuning because to keep all fine-tuning allocation datasets the same size we set them to the size of the smallest subgroup in each dataset. For each dataset, we pre-train three models with different random seeds. For each subgroup allocation experiment, we fine-tune the final classification layer of each pre-trained model on the three randomly generated data splits, resulting in a total of nine fine-tuned models per subgroup allocation. This is then repeated across all the subgroups we consider. This reflects one of the strengths of our work, that compared to the typical data balancing paper, we compare many different subgroups within a dataset (5 to 11), allowing us to gain new insights on why they may have different properties. 

Table~\ref{tab: implementation details} summarises the implementation choices for each dataset. We conduct hyperparameter tuning both for pre-training and fine-tuning. We optimise hyperparameters within the following ranges:

\begin{itemize}
    \item \textbf{Batch size}: {64, 128, 256}
    \item \textbf{Learning rate}: [1e-6:1e-3]
    \item \textbf{Weight decay}: [1e-6:1e-2] 
\end{itemize}

\begin{table}[htb]
\caption{Implementation details for all models.}
\label{tab: implementation details}
\begin{small}
\begin{center}
\resizebox{\columnwidth}{!}{%
\begin{tabular}{@{}lllll@{}}
\toprule
\textbf{Training strategy}  & \textbf{MNIST}           & \textbf{MIMIC-CXR} & \textbf{HAM10000} & \textbf{Civil\textunderscore comments}                \\ \midrule
\textbf{$Y$}         & Even/odd digit                                   & Pleural effusion  & Malignant/benign lesion & Comment toxicity                                    \\ \midrule
\textbf{Backbone}         & 2-layer CNN                                     & DenseNet121 \cite{densenet} & ViTB16 \cite{dosovitskiy2021imageworth16x16words} & BERTClassifier (uncased) \cite{bertclassifier}                                    \\ \midrule
\textbf{Initialisation}         & None                                    & ImageNet \cite{dengimagenet2009}  & ImageNet \cite{dengimagenet2009} & Bookcorpus, Wikipedia (English)                                     \\ \midrule
$\mathbf{N_{pretrain}}$        & 5,000                                   & 33,237  & 2,000 & 17,920                                     \\ \midrule
$\mathbf{N_{finetune}}$        & 24,000                                   & 22,353  & 3,184 & 24,000                                     \\ \midrule
$\mathbf{N_{test}}$          & 4,000                                   & 26,590  & 1,000 & 12,243                                    \\ \midrule
\textbf{Image size}         & 3x28x28                                 & 3x256x256 & 3x256x256 & NA                                 \\ \midrule
\textbf{Augmentation}       & Flip, rotation, Gaussian blur    & Flip, rotation, affine transformation, crop & Flip, rotation, color jitter, affine transformation, crop & None \\ \midrule
\textbf{Optimiser}          & Adam                                    & Adam  & AdamW & AdamW                                     \\ \midrule
\textbf{Loss}               & Binary cross-entropy                    & Binary cross-entropy  & Binary cross-entropy & Binary cross-entropy                       \\ \bottomrule
\end{tabular}%
}
\end{center}
\end{small}
\end{table}


We train all models on 80gb NVIDIA A100s. Each model training run takes an average of 1 GPU hour, so all our experiments cost approximately 350 GPU hours.

We create natural subgroups based on the metadata available, such as based on sex, age, ethnicity, or image-related attributes like the position of the X-ray.
subgroups are generated based on whether the comment mentions specific attributes, like gender, religion, or sexual orientation. 

\begin{table}[htb]
\begin{small}
\centering
\caption{Subgroups used for each dataset and their base population prevalence and class prevalence.}
\label{tab:subgroup-attributes}
\setlength{\tabcolsep}{6pt}
\renewcommand{\arraystretch}{1.15}
\begin{tabular}{@{}lllcc@{}}
\toprule
\textbf{Dataset} & \textbf{Attribute} & \textbf{Group 0 / 1} & $\mathbf{P(A{=}1)}$ & $\mathbf{P(Y|A{=}0) / P(Y|A{=}1)}$ \\
\midrule
\multirow{3}{*}{\textbf{MNIST}} 
  & Colour & Red / Green & 0.50 & \textcolor{black}{0.50 / 0.50} \\
  & Digit value & Below 5 / Over & 0.50 & \textcolor{black}{0.50 / 0.50} \\
  & Random & / & 0.50 & \textcolor{black}{0.50 / 0.50} \\
\midrule
\multirow{11}{*}{\textbf{MIMIC-CXR}} 
  & View Position & Lateral / Frontal & 0.65 & 0.16 / 0.26 \\
  & Patient Orientation & Recumbent / Erect & 0.77 & 0.35 / 0.20 \\
  & Procedure & Portable / Fixed scanner & 0.46 & 0.35 / 0.11 \\
  & Support\_Devices & Absent / Present & 0.23 & 0.17 / 0.39 \\
  & Gender & Female / Male & 0.51 & 0.20 / 0.25 \\
  & Insurance & Not-private / Private & 0.27 & 0.24 / 0.19 \\
  & Language & English / Non-English & 0.88 & \textcolor{black}{0.22 / 0.22} \\
  & Marital Status & Married / Unmarried & 0.44 & 0.20 / 0.25 \\
  & Race & White / Non-White & 0.66 & 0.17 / 0.25 \\
  & Age & Under 60 / Above 60 & 0.56 & 0.16 / 0.28 \\
  & Random & / & 0.50 & \textcolor{black}{0.22 / 0.23} \\
\midrule
\multirow{5}{*}{\textbf{HAM10000}} 
  & Sex & Male / Female & 0.54 & 0.11 / 0.22 \\
  & Age & Under 50 / Above 50 & 0.46 & 0.08 / 0.27 \\
  & Dataset of origin & Rosendahl or Vidir\_molemax / Neither & 0.62 & 0.23 / 0.13 \\
  & Localisation of lesion & Central / Extremity & 0.38 & \textcolor{black}{0.16 / 0.17} \\
  & Random & / & 0.53 & \textcolor{black}{0.16 / 0.17} \\
\midrule
\multirow{6}{*}{\textbf{Civil\_comments}} 
  & Gender & No mention / Mention & 0.20 & 0.10 / 0.15 \\
  & Orientation & No mention / Mention & 0.03 & 0.11 / 0.26 \\
  & Religion & No mention / Mention & 0.16 & \textcolor{black}{0.11 / 0.13} \\
  & Race & No mention / Mention & 0.09 & 0.10 / 0.26 \\
  & Year of comment & Pre-2017 / Post-2017 & 0.76 & \textcolor{black}{0.11 / 0.11} \\
  & Random & / & 0.50 & \textcolor{black}{0.11 / 0.11} \\
\bottomrule
\end{tabular}
\end{small}
\end{table}

\clearpage
\section{Supplementary results on subgroup allocation sensitivity}
\label{sec: appendix-subgroup-sensitivity}

\subsection{Sensitivity of all attributes}
We report full results for subgroup allocation sensitivity across all attributes for two full fine-tuning and last-layer fine-tuning (Tables~\ref{table:fullftsubgroupslopes} and ~\ref{table:llrtsubgroupslopes} respectively).

\begin{table}[h!]
\caption{Sensitivity of subgroup performance to fine-tuning allocation across datasets and attributes in full model fine-tuning. Reported values are esimated regression slopes (mean $\pm$ standard deviation across nine fine-tuning runs) for subgroup loss, 
balanced accuracy, and AUC as allocation varies from 0\% to 100\%. Larger absolute slopes indicate higher sensitivity.}
\label{table:fullftsubgroupslopes}
\centering
\begin{small}
\begin{tabular}{llrrr}
\toprule
Dataset & Attribute & Loss slope & Balanced Acc slope & AUC slope \\
\midrule
MIMIC & View Position                                & -0.027 (0.001) & 0.034 (0.003) & 0.012 (0.001) \\
MIMIC & Patient Orientation & -0.009 (0.001) & 0.013 (0.002) & 0.003 (0.000) \\
MIMIC & Procedure         & -0.026 (0.001) & 0.026 (0.002) & 0.016 (0.001) \\
MIMIC & Support\_Devices                            & -0.009 (0.001) & 0.010 (0.001) & 0.006 (0.000) \\
MIMIC & Gender                                      & -0.003 (0.001) & 0.004 (0.002) & 0.002 (0.000) \\
MIMIC & Insurance                                   & -0.002 (0.001) & 0.003 (0.002) & 0.001 (0.000) \\
MIMIC & Language                                    & -0.002 (0.001) & 0.001 (0.002) & 0.001 (0.000) \\
MIMIC & Marital\_Status                             & -0.001 (0.001) & -0.002 (0.001) & 0.000 (0.000) \\
MIMIC & Race\_cat                                   & -0.002 (0.001) & -0.003 (0.002) & 0.001 (0.000) \\
MIMIC & Age                                         & -0.006 (0.001) & 0.009 (0.002) & 0.003 (0.000) \\
MIMIC & Random                                      & -0.000 (0.001) & 0.000 (0.002) & 0.000 (0.000) \\
\midrule
HAM10000 & Sex          & -0.090 (0.008) & 0.058 (0.003) & 0.032 (0.001) \\
HAM10000 & Age          & -0.157 (0.011) & 0.049 (0.003) & 0.044 (0.001) \\
HAM10000 & Dataset      & -0.230 (0.011) & 0.129 (0.003) & 0.077 (0.002) \\
HAM10000 & Localization & -0.173 (0.010) & 0.082 (0.003) & 0.049 (0.001) \\
HAM10000 & Random       &  0.018 (0.008) & -0.010 (0.003) & -0.002 (0.001) \\
\midrule
Civil\_comments & Gender      & -0.034 (0.011) & 0.023 (0.003) & 0.017 (0.003) \\
Civil\_comments & Orientation & -0.243 (0.012) & 0.094 (0.004) & 0.089 (0.003) \\
Civil\_comments & Religion    & -0.048 (0.012) & 0.031 (0.003) & 0.025 (0.003) \\
Civil\_comments & Race        & -0.146 (0.014) & 0.037 (0.002) & 0.041 (0.002) \\
Civil\_comments & Year        & -0.022 (0.013) & 0.013 (0.003) & 0.011 (0.004) \\
Civil\_comments & Random      & -0.014 (0.011) & 0.008 (0.003) & 0.007 (0.003) \\
\bottomrule
\end{tabular}
\end{small}
\end{table}

\begin{table}[h!]
\caption{Sensitivity of subgroup performance to fine-tuning allocation across datasets and attributes in last-layer fine-tuning. Reported values are esimated regression slopes (mean $\pm$ standard deviation across nine fine-tuning runs) for subgroup loss, 
balanced accuracy, and AUC as allocation varies from 0\% to 100\%. Larger absolute slopes indicate higher sensitivity.}
\label{table:llrtsubgroupslopes}
\centering
\begin{small}
\begin{tabular}{llrrr}
\toprule
Dataset & Attribute & Loss slope & Balanced Acc slope & AUC slope \\ \midrule
MIMIC & View Position                                & -0.007 (0.001) & 0.001 (0.002) & -0.000 (0.000) \\
MIMIC & Patient Orientation & -0.008 (0.001) & -0.002 (0.001) & 0.000 (0.000) \\
MIMIC & Procedure           & -0.011 (0.001) & 0.014 (0.002) & 0.003 (0.000) \\
MIMIC & Support\_Devices                            & -0.005 (0.001) & 0.000 (0.001) & 0.000 (0.001) \\
MIMIC & Gender                                      & -0.000 (0.001) & 0.001 (0.002) & 0.000 (0.000) \\
MIMIC & Insurance                                   & -0.000 (0.001) & 0.000 (0.002) & -0.000 (0.000) \\
MIMIC & Language                                    & -0.001 (0.001) & -0.000 (0.001) & 0.001 (0.000) \\
MIMIC & Marital\_Status                             & -0.000 (0.001) & -0.000 (0.002) & -0.000 (0.000) \\
MIMIC & Race\_cat                                   & -0.000 (0.001) & -0.002 (0.002) & -0.000 (0.000) \\
MIMIC & Age                                         & -0.002 (0.001) & 0.002 (0.002) & 0.001 (0.000) \\
MIMIC & Random                                      &  0.000 (0.001) & 0.001 (0.002) & -0.000 (0.000) \\
\midrule
HAM10000 & Sex          & -0.009 (0.002) & 0.004 (0.003) & 0.005 (0.001) \\
HAM10000 & Age          & -0.032 (0.002) & -0.011 (0.003) & 0.004 (0.001) \\
HAM10000 & Dataset      & -0.023 (0.005) & 0.020 (0.003) & 0.010 (0.001) \\
HAM10000 & Localization & -0.014 (0.003) & 0.010 (0.003) & 0.004 (0.001) \\
HAM10000 & Random       &  0.000 (0.003) & -0.001 (0.003) & -0.000 (0.001) \\
\midrule
Civil\_comments & Gender      & -0.002 (0.002) & 0.001 (0.002) & 0.000 (0.002) \\
Civil\_comments & Orientation & -0.026 (0.002) & -0.000 (0.002) & 0.002 (0.002) \\
Civil\_comments & Religion    & -0.003 (0.002) & -0.000 (0.002) & 0.000 (0.002) \\
Civil\_comments & Race        & -0.022 (0.002) & 0.004 (0.002) & 0.000 (0.002) \\
Civil\_comments & Year        & -0.000 (0.002) & 0.000 (0.003) & 0.000 (0.002) \\
Civil\_comments & Random      & -0.000 (0.002) & -0.001 (0.002) & -0.000 (0.002) \\
\bottomrule
\end{tabular}
\end{small}
\end{table}

\subsection{Experimenting with more complex fits}
\label{sec:complex-fits}
We also experiment with the subgroup allocation scaling law model introduced by \citet{rolfRepresentationMattersAssessing2021a}, which describes subgroup loss as a function of group size and total sample size:  
\[
\ell_g(n_g, n) \;=\; \sigma_g^2 \, n_g^{-p_g} \;+\; \tau_g^2 \, n^{-q_g} \;+\; \delta_g,
\]
where $n_g$ is the number of samples in subgroup $g$, $n$ is the total number of samples, and $\sigma_g, \tau_g > 0$, $p_g, q_g > 0$. Following their setup, we conduct additional experiments where we vary the fine-tuning dataset size $K$ (to obtain more data points) and fit this model under the same parameter constraints. 

In our experiments, however, we find that the fitted parameters often come with very large standard deviations, suggesting instability in the estimates. We also find that estimated parameters that are highly non-robust to small changes in how the fit is estimated or which data points are included. An example for two groups in MIMIC is given in Table~\ref{tab: rolf model fits}, where we have followed the exact same procedure as \citet{rolfRepresentationMattersAssessing2021a}. This is most likely because in \citet{rolfRepresentationMattersAssessing2021a}, all groups showed some level of performance improvement with increased allocation, thus showing a clear power-law trend. In contrast, in our case, some subgroups show no improvement in allocation (like Gender in MIMIC), making it hard to reliably fit a power-law model. This highlights an advantage of linear fits, which is that they reliably and consistently capture the first-order pattern, telling  us whether performance changes across allocations and how much it changes. This is sufficient to capture a strong and consistent relationship to latent representation separation. 

\begin{table}[h!]
\caption{Estimated power-law scaling model parameter fits and standard deviations.}
\label{tab: rolf model fits}
\centering
\begin{small}
\begin{tabular}{llrrrrr}
\toprule
Attribute & Group & $\sigma_g $ & $p_g $ & $\tau_g $ & $q_g $ & $\delta_g $ \\
\midrule
Gender & 0 & 0.000 $\pm$ 0.000 & 1.298 $\pm$ 0.000 & 3136.744 $\pm$ 13745.833 & 2.000 $\pm$ 0.964 & 0.302 $\pm$ 0.020 \\
Gender & 1 & 278.269 $\pm$ 647.492 & 2.000 $\pm$ 0.650 & 3002.106 $\pm$ 13853.293 & 2.000 $\pm$ 1.015 & 0.348 $\pm$ 0.019 \\
ViewPosition & 0 & 4.843 $\pm$ 6.518 & 0.850 $\pm$ 0.387 & 1.523 $\pm$ 4.728 & 0.225 $\pm$ 1.188 & 0.000 $\pm$ 1.388 \\
ViewPosition & 1 & 1.035 $\pm$ 0.806 & 0.260 $\pm$ 0.336 & 2.175 $\pm$ 8.930 & 0.276 $\pm$ 1.359 & 0.000 $\pm$ 1.609 \\
\bottomrule
\end{tabular}
\end{small}
\end{table}

\subsection{Measuring subgroup allocation sensitivity at the extremes}
\label{sec:allocation at the extremes}
We further probe the robustness of linear fits by specifically analysing subgroup allocation sensitivity at the extremes. We notice that in some cases, the slope of performance change is stronger at extreme allocations, i.e. 0 to 10\% or 90 to 100\% (for example, in Figure \ref{fig:mnist_reps_allocs}). We therefore restrict our analysis to performance changes between 0 and 10\% allocations and between 90\% and 100\% and obtain loss slopes specifically at these ranges. We find that, as expected, these slopes are on average about 2 times the magnitude of those obtained on the full allocation space, but that both are values are very much correlated, with Spearman correlation values between 0.77 and 0.90 (Table \ref{tab: corr extreme allocs}. Most importantly, the correlation between representation separation and sensitivity to subgroup allocation is present in all three datasets, and statistically significant for MIMIC and Civil\_comments (Figure \ref{fig:correlation extreme allocs}). This adds to the robustness of our results, showing that our hypothesis holds across subgroup allocations, whether that across the full subgroup distribution, or only at the thresholds. 

\begin{table}[t]
\centering
\caption{
Correlation between full-range and extreme-allocation sensitivity estimates. Sensitivity is measured as the mean slope of the loss across the 0-10\% and 90-100\% allocation range (the extremes) and compared to the loss slope across the full allocation range.
}
\label{tab: corr extreme allocs}
\label{tab:extreme_full_slope_correlation}
\begin{tabular}{lccc}
\toprule
\textbf{Dataset} & \textbf{No. of groups} & \textbf{Spearman $\rho$} &\textbf{ $p$-value} \\
\midrule
MIMIC & 10 & 0.77 & 0.009 \\
HAM10000 & 5 & 0.90 & 0.037 \\
CivilComments & 6 & 0.83 & 0.042 \\
\bottomrule
\end{tabular}
\end{table}

\begin{figure}[h!]
    \centering

    \includegraphics[width=1\linewidth]{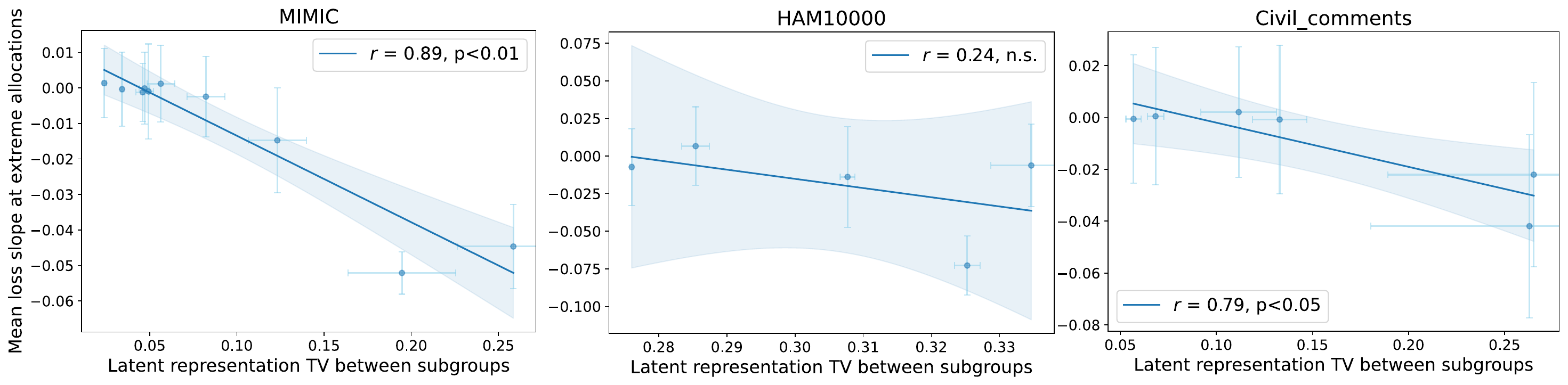}
    \caption{Sensitivity to subgroup allocation at the thresholds is highly correlated with separation in the pre-trained model's representation space (as measured by total variation distance, TV) across the three datasets. Each dot represents mean TV and loss slope (between 0 and 10\% and 90 and 100\% subgroup allocation) for one subgroup, averaged across 9 fine-tuning runs, with bars corresponding to standard deviations, and Pearson correlation also shown.}
    \label{fig:correlation extreme allocs}
\end{figure}

\clearpage
\section{Supplementary results on common hypotheses not holding}
\label{sec: appendix-common-hypotheses}

We test three common hypotheses which may explain when subgroup performance depends on their data distribution. These are:
\begin{enumerate}
    \item Models may be more sensitive to the allocation of subgroups which are initially under-represented in the pre-training dataset, such that if they are fine-tuned on more data from the under-represented groups, subgroup performance is expected to increase.
    \item Models may be more sensitive to the allocation of subgroups on which they initially perform poorly (perhaps because those groups are inherently noisier or more difficult), such that if they are fine-tuned on more data from the disadvantaged groups, subgroup performance is expected to increase.
    \item Models may be more sensitive to the allocation of subgroups with high class imbalance because modifying the allocation will also lead to changes in $\mathcal{P}(Y)$
\end{enumerate}

We see no consistent correlation between the each of those variables and sensitivity to subgroup allocation (slope of loss across allocations, i.e., performance change), as shown in Figure \ref{fig: full alt hypotheses no normalisation}. Figure \ref{fig:correlation subgroup prevalence accuracy slope} is a condensed version of Figure \ref{fig: full alt hypotheses no normalisation}, where all datasets are plotted together, and the loss slope is normalised such that each result is on the same scale. 

\begin{figure}[h!]
    \centering
    
    \includegraphics[width=1\linewidth]{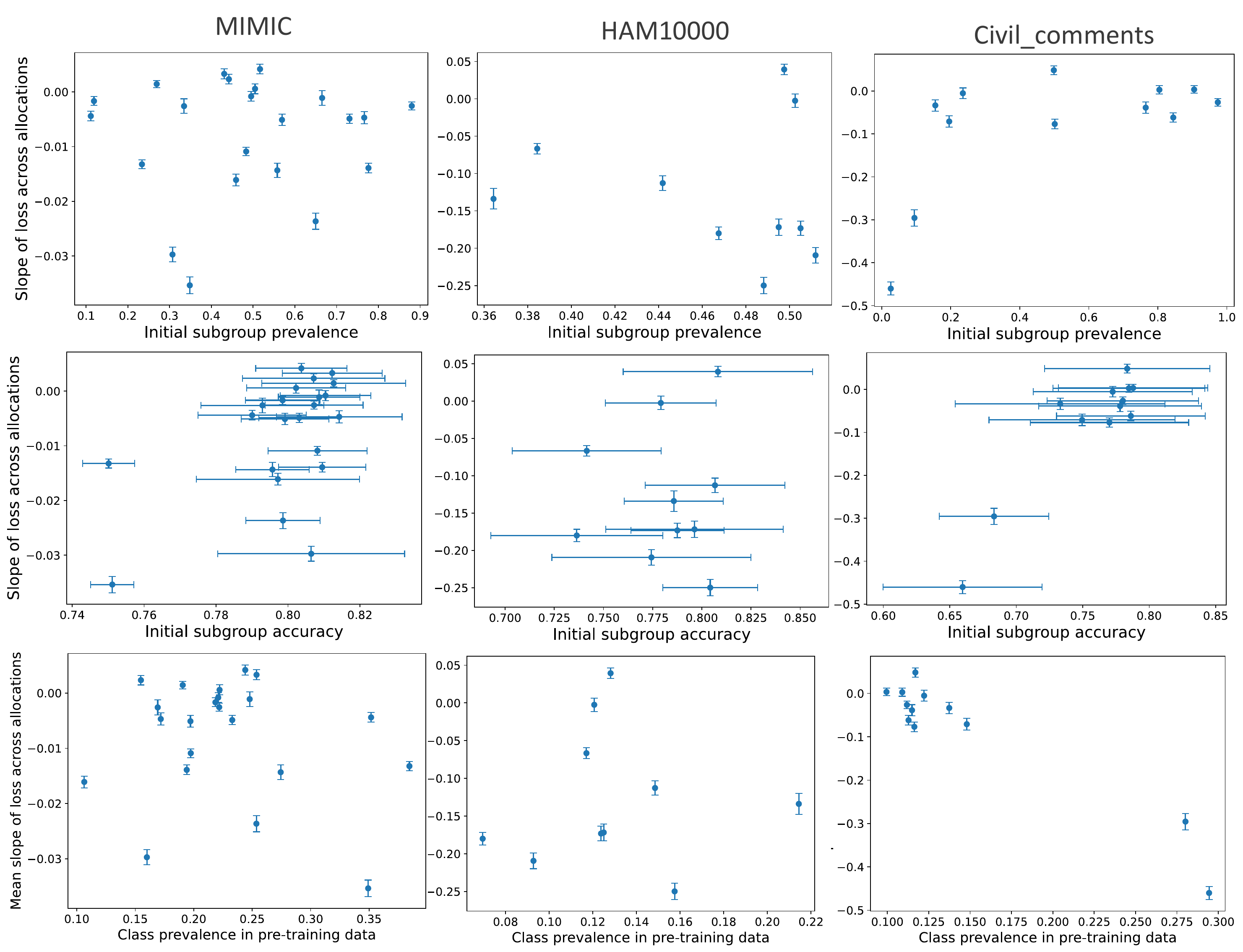}
    \caption{Performance on subgroups under-represented during pre-training (top), disadvantaged (middle), and with class imbalance (bottom) does not necessarily improve with increasing dataset allocation. No clear correlation to performance change (a negative gradient of subgroup loss change indicates performance improvement) is observed in any of the three settings. Each dot represents one subgroup with bars indicating variation across 9 fine-tuning runs.}
    \label{fig: full alt hypotheses no normalisation}
\end{figure}

\begin{table}[ht]
    \centering
    \caption{Correlation results across datasets and hypotheses.}
    \label{appendix-tab:correlation_alt_hypotheses}
    \begin{tabular}{llrrcr}
        \toprule
        Dataset & Hypothesis & Correlation & $p$-value & Significant ($p<0.05$) & $N$ \\
        \midrule
        MIMIC & Subgroup prevalence & 0.0586 & 0.7955 & False & 22 \\
        MIMIC & Subgroup accuracy   & 0.5197 & 0.01318 & True  & 22 \\
        MIMIC & Class imbalance     & -0.1906 & 0.3954 & False & 22 \\
        \midrule
        HAM10000 & Subgroup prevalence & -0.1027 & 0.7777 & False & 10 \\
        HAM10000 & Subgroup accuracy   & 0.0616  & 0.8658 & False & 10 \\
        HAM10000 & Class imbalance     & 0.0261  & 0.9430 & False & 10 \\
        \midrule
        Civil\_comments & Subgroup prevalence & 0.5901  & 0.05339 & False & 12 \\
        Civil\_comments & Subgroup accuracy   & 0.9281  & $1.341 \times 10^{-5}$ & True & 12 \\
        Civil\_comments & Class imbalance     & -0.9510 & $2.058 \times 10^{-6}$ & True & 12 \\
        \bottomrule
    \end{tabular}
\end{table}
\clearpage
\section{Full theoretical results}
\label{appendix:full-theory}
\subsection{Full proofs}
\label{appendix:full-proofs}

\begin{assumption}
\label{assumption:setup}
Throughout this section, we assume the following setup:
\begin{enumerate}
\item \textbf{Fixed representation:} A pre-trained model $h$ generates fixed representations $Z = h(X)$ with which 
binary last-layer classifiers $g_\theta$ are fine-tuned, such that $f_{\theta}(x)=g_\theta(h(x))$.
\item \textbf{Different fine-tuning distributions:} We compare two models $g_{\theta'}$ and $g_{\theta''}$ fine-tuned on datasets $\mathtt D'$ and $\mathtt D''$ which differ only in their subgroup allocations $\alpha^{(A)}$.
\item \textbf{Fixed test distribution:} All models are evaluated on a fixed test distribution $\mathbb{Q}$ over inputs $X$, labels $Y\in\{0,1\}$, and sensitive attributes $A\in\{0,1\}$.
\end{enumerate}
\end{assumption}

We now prove the main theoretical results.
\MainLemma*

\begin{proof}
\label{lemmaproof}
We only prove the case for $A=1$, since the case for $A = 0$ is analogous. For simplicity, write
\begin{align*}
&\mu_a(\cdot\mid y):=\mathbb Q\bigl(Z\,\big|\, Y=y,A=a\bigr)\\\qquad
&\mu(\cdot\mid y):=\mathbb Q\bigl(Z\,\big|\, Y=y\bigr)\\
& \pi_y:=\mathbb Q(A=1\mid Y=y)
\end{align*}
By the law of total probability, we have that
\begin{equation}
\label{law:total_prob}
\mu(\cdot\mid y)=\pi_y\,\mu_1(\cdot\mid y)+(1-\pi_y)\,\mu_0(\cdot\mid y).
\end{equation}
View $\mu_0(\cdot\mid y)$ and $\mu_1(\cdot\mid y)$ as probability measures on the same measurable space. By using \eqref{law:total_prob}, we have that
\[
\mu_1(\cdot\mid y)-\mu(\cdot\mid y)
= \mu_1(\cdot\mid y) - \bigl(\pi_y\mu_1(\cdot\mid y)+(1-\pi_y)\mu_0(\cdot\mid y)\bigr)
= (1-\pi_y)\,\bigl(\mu_1(\cdot\mid y)-\mu_0(\cdot\mid y)\bigr).
\]
Taking total variation norms and using homogeneity of total variation for signed measures,
\[
\mathrm{TV}\!\bigl(\mu_1(\cdot\mid y),\mu(\cdot\mid y)\bigr)
= (1-\pi_y)\,\mathrm{TV}\!\bigl(\mu_1(\cdot\mid y),\mu_0(\cdot\mid y)\bigr).
\]
Then, it holds
\begin{equation}
\label{eq:z_cont}
\mathrm{TV}\!\bigl(\mu_1(\cdot\mid y),\mu(\cdot\mid y)\bigr)
\le (1-\pi_y)\,\varepsilon
\leq\varepsilon,
\end{equation}
where we have used that $1-\pi_y \leq 1$. The claim follows directly from \eqref{eq:z_cont}.
\end{proof}
We can use this lemma to prove the main result. 
\MainThm*
\begin{proof}
\label{theoremproof}

We define accuracy as the model's balanced accuracy over positive and negative samples:
\begin{align}
\label{eq:acc}
\textsc{Acc}_{\theta}=\frac{1}{2}\sum_y \mathbb Q(f_{\theta}(X)=y\mid Y=y).
\end{align}
We can write the subgroup balanced accuracy via conditioning on the true label as
\begin{align}
\label{eq:acc_1}
\textsc{Acc}_{\theta}(A=a) =\frac{1}{2}\sum_y \mathbb Q(f_{\theta}(X)=y\mid Y=y, A=a).\
\end{align}

Define the sets $S'(y) := \{z \colon g_{\theta'}(z) = y\}$ and $S''(y) := \{z \colon g_{\theta''}(z) = y\}$. Then it holds
\begin{align}
& \absl{\left ( \textsc{Acc}_{\theta'}(A=a) - \textsc{Acc}_{\theta''}(A=a)\right ) } \nonumber \\ 
& = \frac{1}{2}\absl{\sum_y \left (\mathbb Q(f_{\theta'}(X)=y\mid Y=y, A=a) - \mathbb Q(f_{\theta''}(X)=y\mid Y=y, A=a) \right )} \nonumber \\
& = \frac{1}{2} \absl{\sum_y \mathbb Q(Z\in S'(y)\mid Y=y, A=a) - \mathbb Q(Z \in S''(y)\mid Y=y, A=a) } \label{eq:acc_4}
\end{align}
Furthermore, using the triangle inequality and Lemma \ref{lemma:tv}, it holds
\begin{align}
 &\frac{1}{2} \absl{\sum_y \mathbb Q(Z \in S'(y)\mid Y=y, A=a) - \mathbb Q(Z \in S''(y)\mid Y=y, A=a) }\nonumber \\
 & \leq \frac{1}{2}\sum_y \absl{\mathbb Q(Z\in S'(y)\mid Y=y, A=a) - \mathbb Q(Z \in S'(y)\mid Y=y) }\nonumber \\
 & + \frac{1}{2}\sum_y \absl{ \mathbb Q(Z \in S''(y)\mid Y=y, A=a) - \mathbb Q(Z\in S''(y)\mid Y=y)}\nonumber \\
 & + \frac{1}{2}\absl{ \sum_y \mathbb Q(Z \in S'(y)\mid Y=y) - \mathbb Q(Z \in S''(y)\mid Y=y)}\nonumber \\
 & \leq 2\varepsilon + \frac{1}{2} \absl{ \sum_y \mathbb Q(f_{\theta'}(X)=y\mid Y=y) - \mathbb Q(f_{\theta''}(X)=y\mid Y=y)}  \label{eq:acc_5}
\end{align}
The claim follows by combining \eqref{eq:acc_4} with \eqref{eq:acc_5}, and noting that it holds
\begin{equation*}
\frac{1}{2}\absl{ \sum_y \mathbb Q(f_{\theta'}(X)=y\mid Y=y) - \mathbb Q(f_{\theta''}(X)=y\mid Y=y) }= \absl{{\textsc{Acc}}_{\theta'} - {\textsc{Acc}}_{\theta''} }.
\end{equation*}
\end{proof}

\subsection{Further remarks on theory assumptions and bounds}

\subsubsection{Extending to multiple discrete groups}
We note that our Theorem~\ref{theorem 1} can also be extended to cases with non-binary attributes, i.e., for $A\in\{1,\dots,K\}$. Define
\[
\mathsf{TV}_K(\mathbb Q):=\mathbb{E}_{y}\!\left[\max_{a,b\in[K]}
\mathsf{TV}\!\big(\mathbb{Q}(Z\mid Y=y,A=a),\ \mathbb{Q}(Z\mid Y=y,A=b)\big)\right].
\]
Under the same assumptions, if $\mathsf{TV}_K(\mathbb Q)\le\varepsilon$, then for all $a\in[K]$,
\[
\bigl|\textsc{Acc}_{\theta'}(A=a)-\textsc{Acc}_{\theta''}(A=a)\bigr|
\le 2\varepsilon+\bigl|\textsc{Acc}_{\theta'}-\textsc{Acc}_{\theta''}\bigr|.
\]

\subsubsection{Upper bound values}
In our experiments, we observe that $\absl{{\textsc{Acc}}_{\theta'} - {\textsc{Acc}}_{\theta''} }$ is neglible. As shown in Figure~\ref{fig:overall acc sd plots}, there is very little variation in overall model performance across fine-tuning allocations, with a mean accuracy standard deviation of 0.016, 0.002, 0.001, and 0.001 for MNIST, MIMIC, HAM, and Civil\_comments respectively.

\begin{figure}[htb!]
    \centering
    
    \includegraphics[width=1\linewidth]{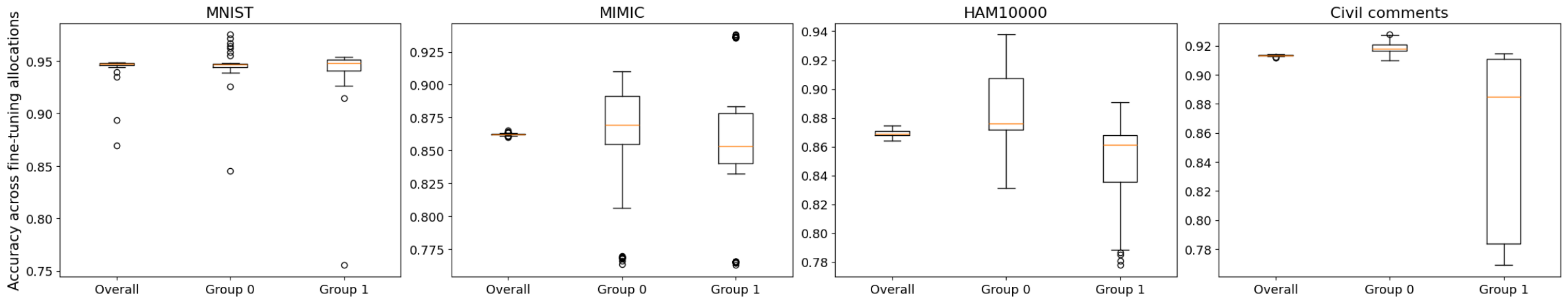}
    \caption{Overall and group-wise balanced accuracy across all fine-tuning runs for MNIST, MIMIC, HAM10000, and Civil\_comments.}
    \label{fig:overall acc sd plots}
\end{figure}

We further calculate the value of the upper bound for each subgroup and compare it to the trival bound of 1. We find that it is generally informative, with most values lying between 0 and 0.25, except for certain very high TV subgroups (as expected) and HAM10000 subgroups. The latter is due to the fact that every HAM10000 subgroup (even random groups) has TV $>0.25$, causing the upper bound to be consistently above 0.5. This is most likely because the HAM10000 dataset is of a much smaller size (1000 images, with only 175 positive samples) which causes the TV estimates to be less reliable and more sensitive to small changes like the bin size used in histogram estimation (as discussed in \S\ref{appendix:tv-limitations}).

\begin{figure}[h!]
    \centering
    
    \includegraphics[width=1\linewidth]{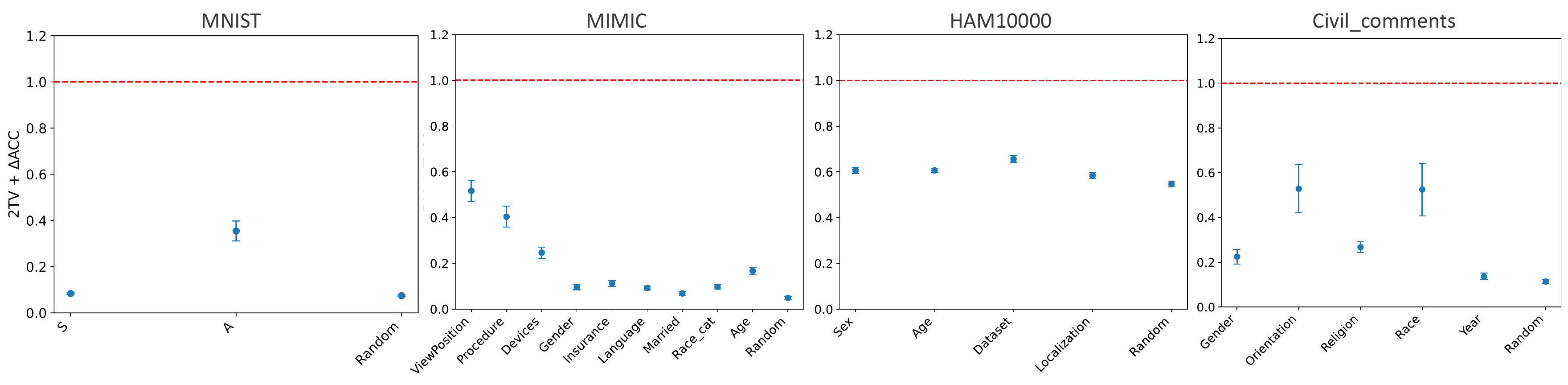}
    \caption{Value of upper bound for MNIST, MIMIC, HAM10000, and Civil\_comments subgroups, with the trivial upper bound (accuracy difference of 1) marked in red.}
    \label{fig:upper bound}
\end{figure}

\clearpage
\section{Supplementary results on representation invariance}
\label{sec: appendix-representation-invariance}

\subsection{Additional details on  distance metrics}
We analyse the pre-trained models' representations by first extracting the penultimate layer embeddings of the test inputs (e.g., 1024-dimensional vector $\vz$ for DenseNet121). We measure the mean \textbf{total variation distance }(TV) between the representations of examples in subgroups $a^{(j)}_0$ and $a^{(j)}_1$. We condition on $Y$ to control for class imbalances. 

Given subgroup embeddings $\vz_{a^{(j)}_0}$ and $\vz_{a^{(j)}_1}$, we estimate their distributions by constructing normalized histograms along each dimension. For a given component, let $p$ and $q$ denote the resulting discrete probability mass functions over bins $b$. The TV is then
$\text{TV}(p,q) = \tfrac{1}{2} \sum_{b} |p(b) - q(b)|,$
where the sum is over histogram bins. In practice, we compute TV for each dimension separately over 50 bins and report the average across dimensions. TV provides a bounded ($[0,1]$) measure of separation. Values near 0 indicate nearly identical marginal distributions, while values near 1 indicate almost complete disjointness. 

\paragraph{Limitations of TV}
\label{appendix:tv-limitations}
While TV is an intuitive and bounded measure of separation, it is known to be limited when measured empirically (from finite samples). This is because estimating TV requires estimating the density of the conditional distribution $Z\mid A$. This estimate can be noisy, especially for small sample sizes where it is often over-estimated, as histograms or kernel density estimates are biased upwards, causing random fluctuations in the smallest group to increase it.

For these reasons, we also explore additional distance metrics including the \textbf{Wasserstein distance} (WD) and the \textbf{Fréchet distance} (FD) which emphasise distinct representation differences.

For WD, we compute the 1D WD along each column of the embeddings matrix and report the mean:
\[
\text{WD}\big(\vz_{a^{(j)}_0}, \vz_{a^{(j)}_1}\big) 
= \frac{1}{k} \sum_{i=1}^{k} W_1\!\big(\vz_{a^{(j)}_0}^{(i)}, \vz_{a^{(j)}_1}^{(i)}\big),
\]
where \(\vz_{a^{(j)}_0}^{(i)}\) and \(\vz_{a^{(j)}_1}^{(i)}\) denote the projections of the embeddings from subgroups \(a^{(j)}_0\) and \(a^{(j)}_1\) onto the \(i\)-th dimension, and \(W_1\) is the univariate Wasserstein-1 distance.

The FD approximates the subgroup feature distributions as multivariate Gaussians. Let \((\mu_{a^{(j)}_0}, \Sigma_{a^{(j)}_0})\) and \((\mu_{a^{(j)}_1}, \Sigma_{a^{(j)}_1})\) be the empirical means and covariances of the embeddings for the two subgroups. Then
\[
\text{FD}\big(\vz_{a^{(j)}_0}, \vz_{a^{(j)}_1}\big) 
= \|\mu_{a^{(j)}_0} - \mu_{a^{(j)}_1}\|_2^2 
+ \operatorname{Tr}\Big(\Sigma_{a^{(j)}_0} + \Sigma_{a^{(j)}_1} - 2(\Sigma_{a^{(j)}_0}^{1/2}\Sigma_{a^{(j)}_1}\Sigma_{a^{(j)}_0}^{1/2})^{1/2}\Big).
\]

Together, these measures provide different estimates of on invariance: TVD emphasizes the largest discrepancies in probability mass between subgroups, WD gives a precise estimate of marginal distribution shift (including differences in distributional shape), and FD captures coarse differences in means and covariances.

\subsection{Representation distances obtained across the four datasets}

We highlight the variation in TV across subgroups in the four datasets in Figure \ref{fig:TV differences all datasets}.

\begin{figure}[h!]
    \centering
    \includegraphics[width=0.9\linewidth]{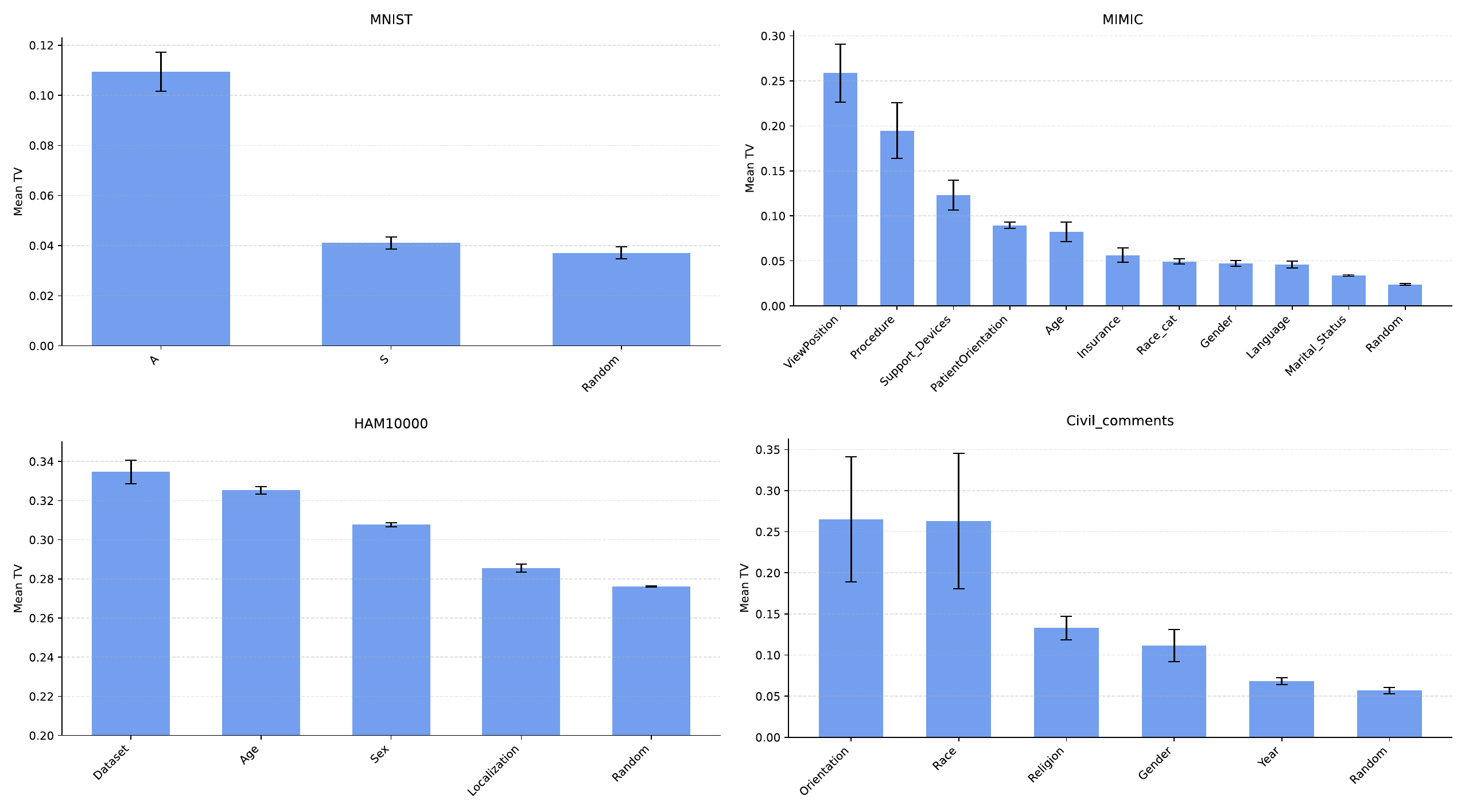}
    \caption{Total variation (TV) distances across subgroups in the pre-trained model's penultimate-layer representation space. We report mean and standard deviation across three random seeds.}
    \label{fig:TV differences all datasets}
\end{figure}

\subsection{Metrics are robust}
All three distance metrics are correlated.

\begin{figure}[h!]
    \centering
    \includegraphics[width=0.6\linewidth]{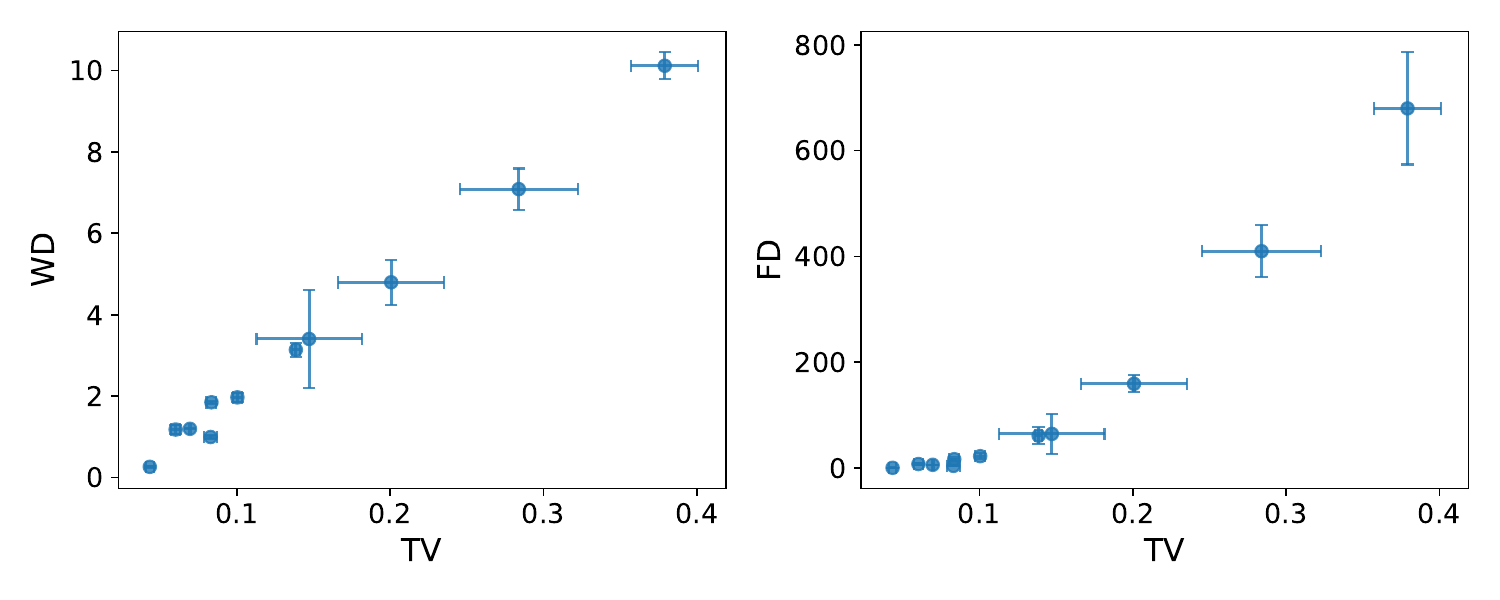}
    \caption{Total variation distance (TV), wasserstein distance (WD), and Fréchet distance (FD) across subgroups appear correlated in MIMIC. Each dot represents a subgroup with error bars representing standard deviation across three pre-training runs.}
    \label{fig:correlation distance metrics}
\end{figure}


\clearpage
\section{Supplementary results on the correlation between representation distance and subgroup allocation sensitivity}
\label{sec: appendix-correlation-distance-sensitivity}
Figure~\ref{fig:full corr distances} extends the main text by showing subgroup allocation sensitivity against additional distance metrics across our datasets. Figure~\ref{fig:llrt auc slope} also shows that latent separation is correlated with sensitivity to subgroup allocation when measured by AUC slope instead of loss. The correlation is strong across metrics, supporting the latent separation hypothesis.  

\begin{figure}[h!]
    \centering
    \includegraphics[width=0.9\linewidth]{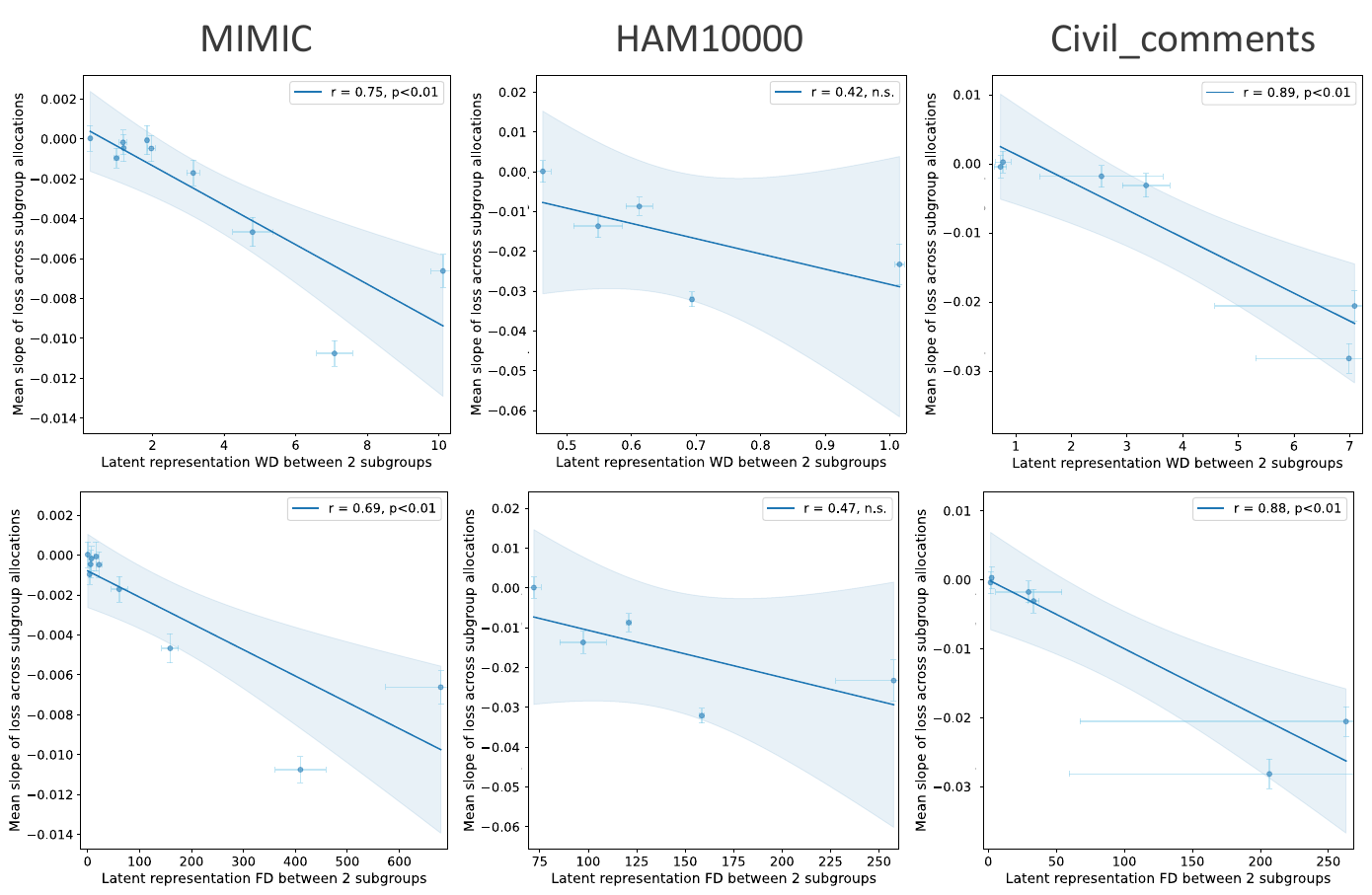}
    \caption{Sensitivity to subgroup allocation is correlated with separation in the pre-trained model's representation space (as measured by wasserstein distance, and Fréchet distance) across the three datasets. Each dot represents mean distance and loss slope for one subgroup, averaged across 9 fine-tuning runs, with bars corresponding to standard deviations, and Pearson correlation coefficients also shown.}
    \label{fig:full corr distances}
\end{figure}

\begin{figure*}[htb!]
    \centering
    \includegraphics[width=1\linewidth]{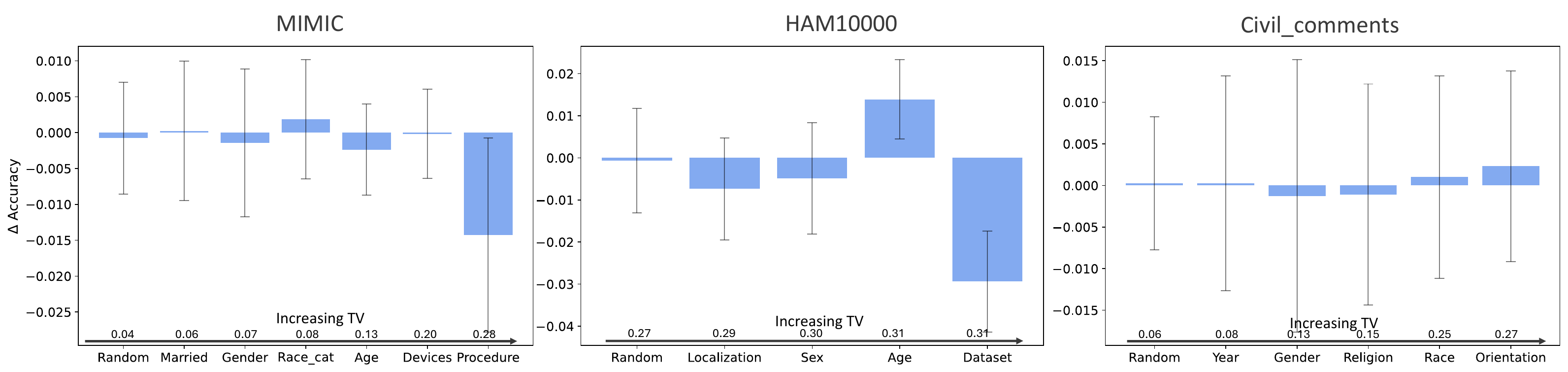}%
    \caption{Accuracy difference when a model has not been trained on a subgroup increases with the separation of the latent representations of the subgroups. Results are shown as mean balanced accuracy difference across 9 fine-tuning runs with error bars indicating standard deviation.}
    \label{fig:llrt-zero-shot-generalisation-represenations-3-datasets}
\end{figure*}

\begin{figure*}[htb!]
    \centering
    \includegraphics[width=1\linewidth]{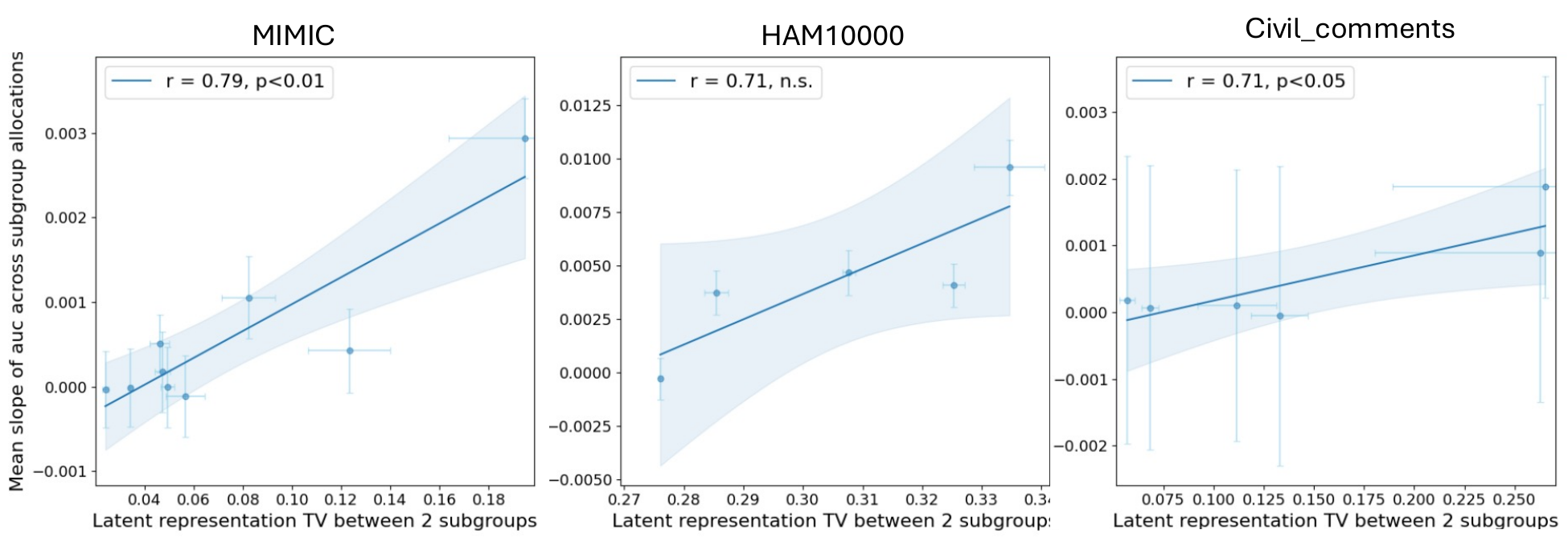}%
    \caption{Sensitivity to subgroup allocation (slope of AUC) is correlated with subgroup TV distance in the pre-trained model's representation space across the three datasets. Each dot represents mean distance and AUC slope for one subgroup, averaged across 9 fine-tuning runs, with bars corresponding to standard deviations, and Pearson correlation coefficients also shown.}
    \label{fig:llrt auc slope}
\end{figure*}


\clearpage
\section{Extending setup to other training regimes}
\label{sec: appendix-training-regimes}

\subsection{Full fine-tuning}
\label{sec: appendix-full-finetuning}

We additionally extend our setup from last-layer fine-tuning to full model fine-tuning. Overall, the correlations are equally significant, and the magnitude of the effects larger (Figures~\ref{fig:correlation slope distances full fine-tuning} and~\ref{fig:generalisation_and_reps_full_fine-tuning}). We attribute this to the fact that in practice, separation of subgroup representations does not change with subgroup allocations (Figure~\ref{fig:reps at varying allocs}). Therefore, while some of our theoretical assumptions still hold, the effect increases as full fine-tuning allows for greater modification of network parameters. 

We hypothesise that representation separation is so consistent between pre-training and fine-tuning and across allocations (Figure \ref{fig:reps at varying allocs}) because of non-spurious shifts in the features necessary to predict $Y$ across groups. For example, in MIMIC, the TV distance between different X-ray views (frontal or lateral) is consistently high. This is most likely because the necessary features to diagnose disease differ depending on the viewpoint, and therefore, across allocations, a model must maintain a high separation in its latent space to perform well on both groups.

\begin{figure}[h!]
    \centering
    \includegraphics[width=1\linewidth]{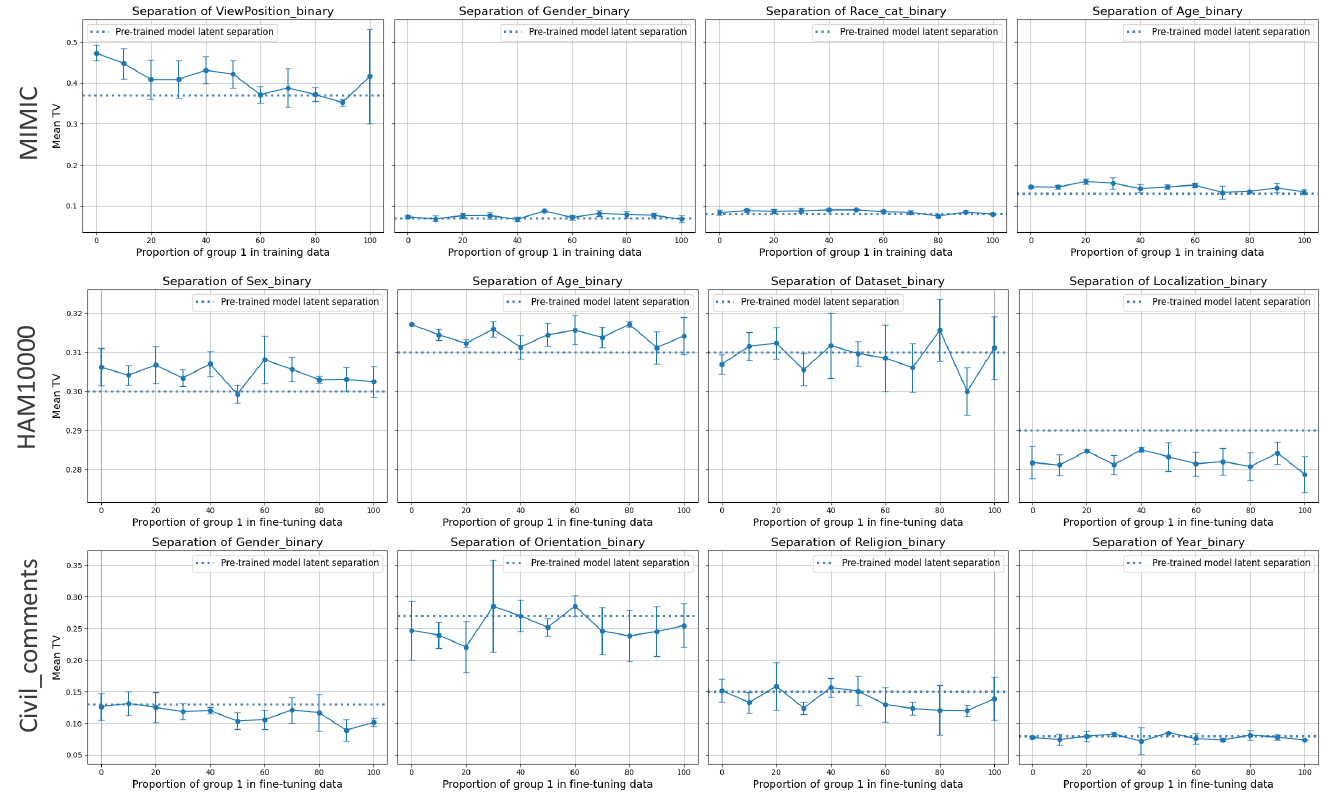}
    \caption{Examples showing latent representation separations are stable across subgroup allocations under full fine-tuning and close to separation of the initial pre-trained model.}
    \label{fig:reps at varying allocs}
\end{figure}

\begin{figure}[h!]
    \centering
    \includegraphics[width=0.9\linewidth]{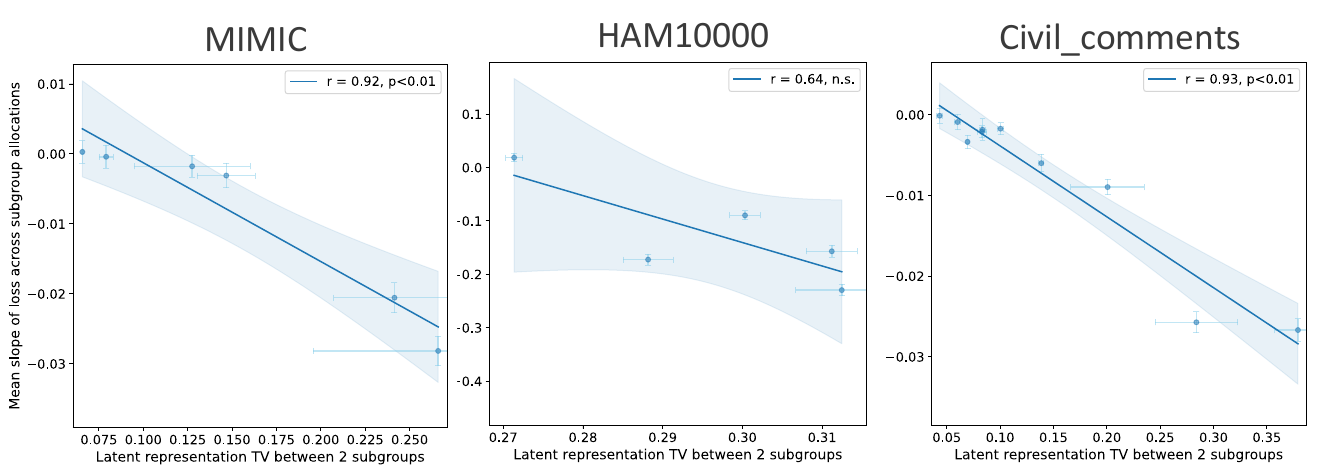}
    \caption{Correlation between loss slope and distance metrics remains strong under full fine-tuning. Each dot represents mean total variation distance and loss slope for one subgroup, averaged across 9 fine-tuning runs, with bars corresponding to standard deviations, and Pearson correlation coefficients also shown.}
    \label{fig:correlation slope distances full fine-tuning}
\end{figure}

\begin{figure}[h!]
    \centering
    \includegraphics[width=1\linewidth]{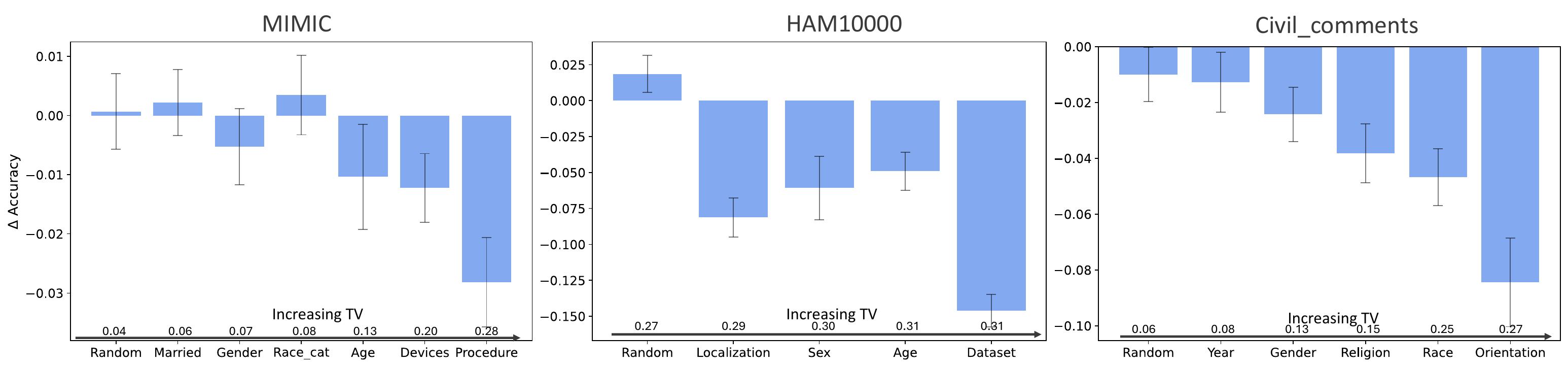}
    \caption{Representation distances continue to predict accuracy gaps at extreme subgroup allocations under full fine-tuning at 100\% allocation vs 0\%. Bars represent mean and standard deviation of the balanced accuracy across 9 fine-tuning runs.}
    \label{fig:generalisation_and_reps_full_fine-tuning}
\end{figure}

\clearpage
\subsection{Training from scratch}
\label{sec:training-from-scratch}
We also test whether the our hypothesis holds when only doing one round of training from scratch. We train the same models from scratch while systematically varying subgroup training data allocation and measure how much performance changes across allocations. We find similar trends in allocation sensitivity, but of a much greater magnitude (e.g., loss slopes are approximately 10x steeper than in last-layer fine-tuning), as expected since the models can vary more. We next explore whether this is linked to latent representation separation, using the model trained on the natural dataset proportions to extract latent embedding vectors. We find very similar results to fine-tuning, with a strong significant correlation between latent representation separation and sensitivity to subgroup allocation (Figures~\ref{fig:correlation slope distances scratch} and \ref{fig:correlation slope distances scratch}). While this finding is less clearly actionable under this setup (e.g., cannot directly inform subsequent data collection efforts), it provides a strong explanation as to the behaviour of models trained under different subgroup allocations.

\begin{figure}[h!]
    \centering
    
    \includegraphics[width=1\linewidth]{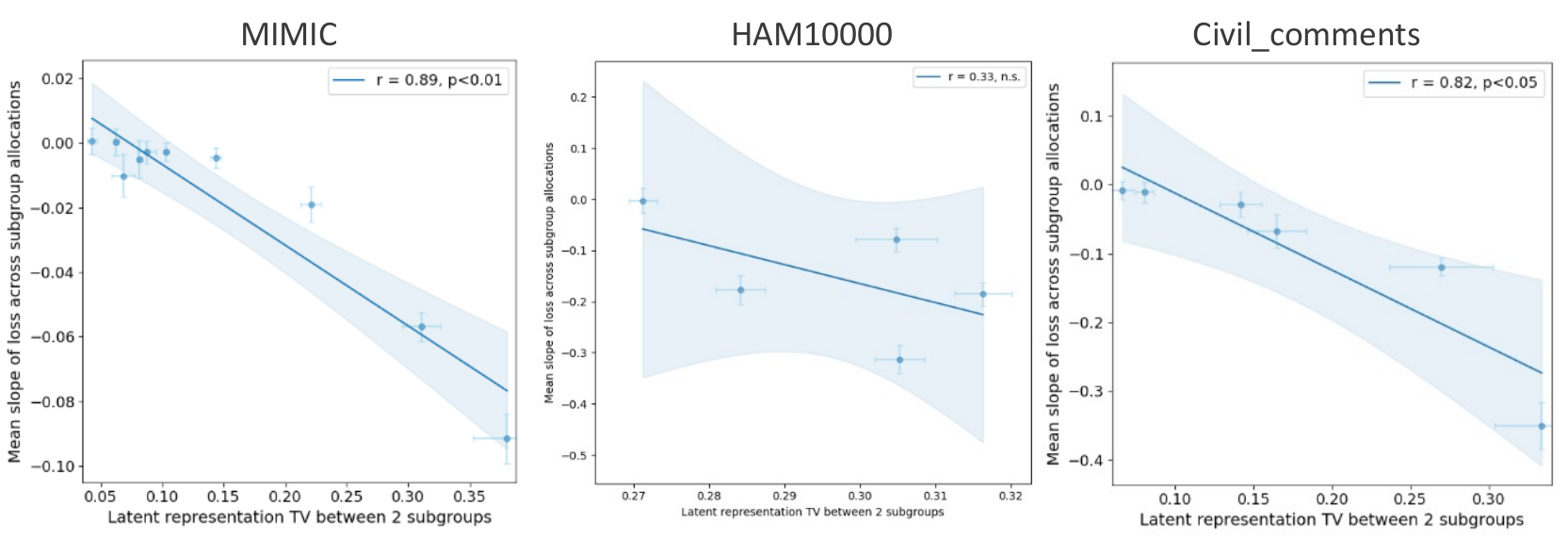}
    \caption{Correlation between loss slope and distance metrics remains strong when training from scratch. Each dot represents mean total variation distance and loss slope for one subgroup, averaged across 3 training runs, with bars corresponding to standard deviations, and Pearson correlation coefficients also shown.}
    \label{fig:correlation slope distances scratch}
\end{figure}


\clearpage
\section{Supplementary results on foundation model fine-tuning}
\label{sec: appendix-vlm-finetuning}
\subsection{Supplementary experimental details}
We test whether our results hold in the much less controlled setting of foundation model fine-tuning. We experiment with two vision-language models, CheXagent \citep{chen2024visionlanguagefoundationmodelenhance}, and RAD-DINO-MAIRA-2 \citep{perezgarcia2024raddino}. Both of these rely on fundamentally different training mechanisms. CheXagent is trained on images and text with SigLIP training while RAD-DINO-MAIRA-2 is only trained on images with a DINO-based setup, providing two distinct embedding regimes to assess the robustness of our hypothesis.

We first pass the MIMIC-CXR images through the image encoders \texttt{XraySigLIP\_\_vit-l-16-siglip-384\_\_webli} and \texttt{rad-dino-maira-2} and obtain 1024- and 768-dimensional embeddings respectively.
We then train a single classification layer on each of the 16,000 MIMIC-CXR training embeddings, varying allocation in the same way as previous experiments. We test the classification model on the MIMIC-CXR test set (which, to the best of our understanding, neither model has been trained on).

\subsection{Supplementary results}
We show differences in latent representation separations in both foundation models in Figure~\ref{fig:TV differences vlm}. 
\begin{figure}[h!]
    \centering
    
    \includegraphics[width=1\linewidth]{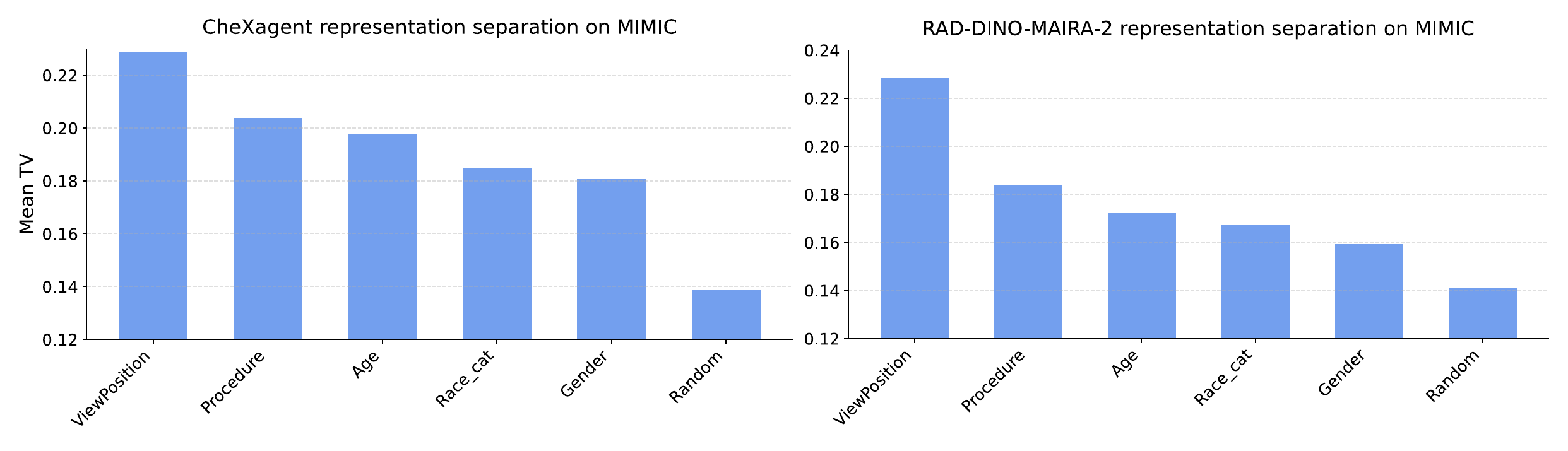}
    \caption{Total variation distances of CheXagent (left) and MAIRA-2 (right) embeddings of MIMIC images across subgroups.}
    \label{fig:TV differences vlm}
\end{figure}

\begin{figure}
    \centering
    
    \includegraphics[width=0.5\linewidth]{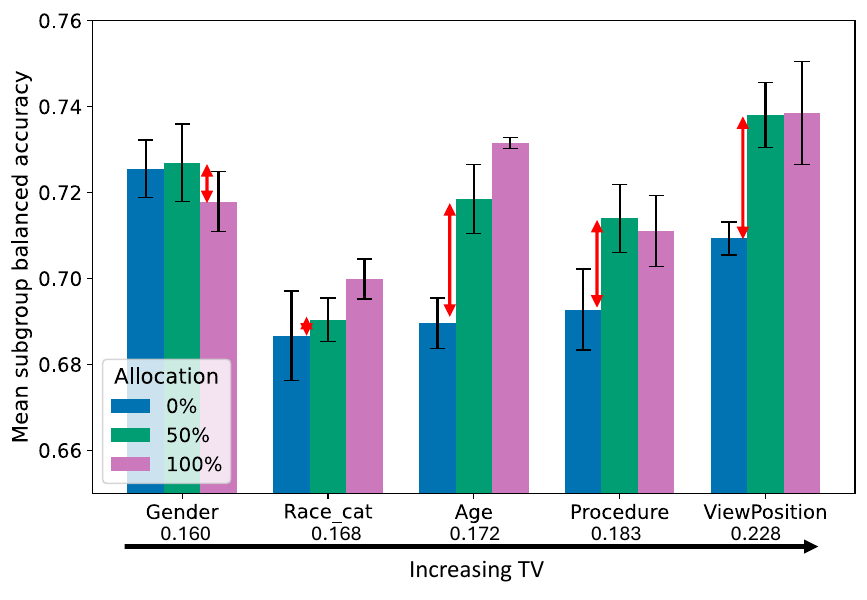}
    \caption{In RAD-DINO-MAIRA-2 foundation model fine-tuning, selecting a balanced allocation for imaging subgroups increases subgroup accuracy by over 0.03, while it has less importance for demographic subgroups, as predicted by their reduced total variation distance. We report mean accuracy and standard deviations for 3 fine-tuning runs, with red arrows for potential gains from balanced allocations.}
    \label{fig: maira-2 acc gains}
\end{figure}

\begin{figure}[h!]
    \centering
    \includegraphics[width=0.9\linewidth]{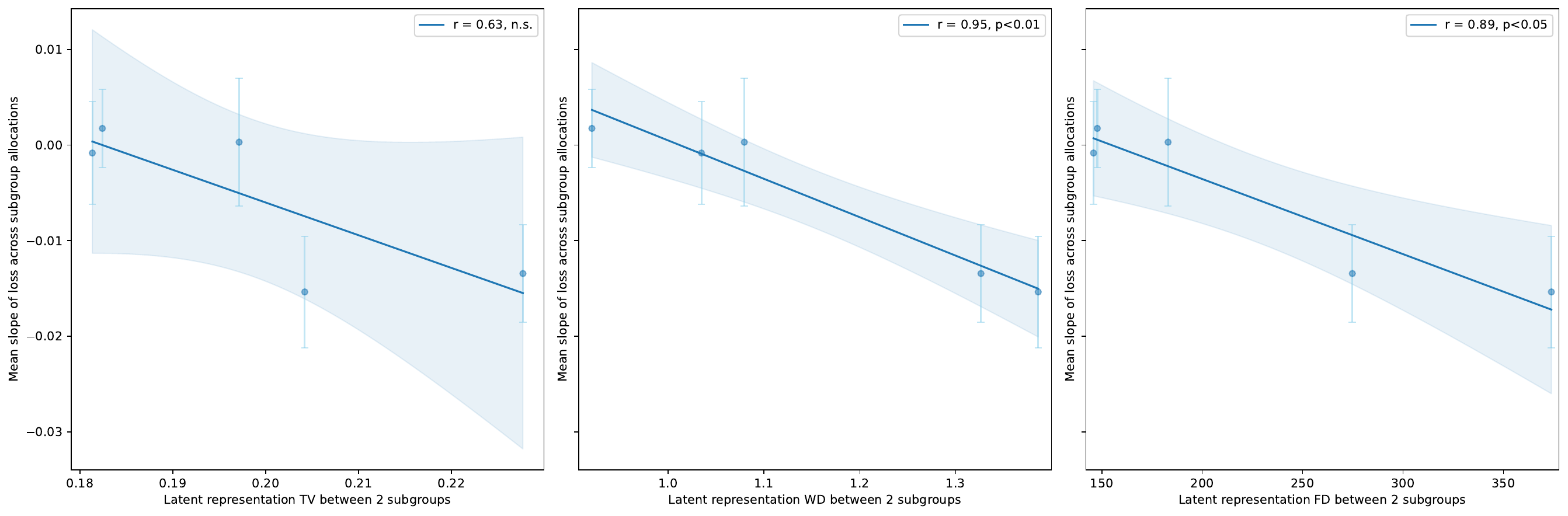}
    \caption{Correlation between latent representation separation and subgroup allocation sensitivity in CheXagent (vision-language model) fine-tuning. Results confirm that our hypothesis generalises beyond task-specific models. Each dot represents mean distance and loss slope for one subgroup, averaged across 9 fine-tuning runs, with bars corresponding to standard deviations, and Pearson correlation coefficients also shown.}
    \label{fig:vlm distances slope}
\end{figure}

\begin{figure}[h!]
    \centering
    
    \includegraphics[width=0.9\linewidth]{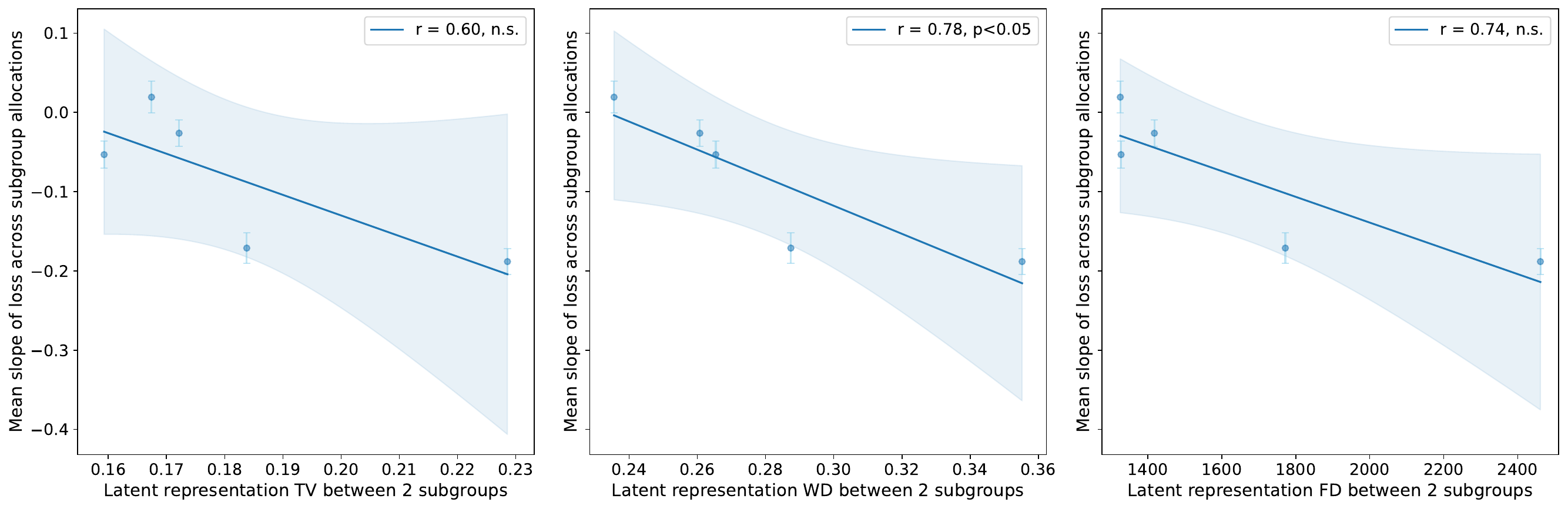}
    \caption{Correlation between latent representation separation and subgroup allocation sensitivity in MAIRA-2 (vision-language model) fine-tuning. Results confirm that our hypothesis generalises beyond task-specific models. Each dot represents mean distance and loss slope for one subgroup, averaged across 9 fine-tuning runs, with bars corresponding to standard deviations, and Pearson correlation coefficients also shown.}
    \label{fig:maira-2 distances slope}
\end{figure}

\subsection{Parameter-efficient fine-tuning}

We further test the robustness of our results by testing our hypothesis in another setup frequently used with foundation models, parameter-efficient fine-tuning (PEFT).  We implement LoRA (with rank 16 and alpha 32) on the CheXagent foundation model, tuning the query and value projection layers and adding a linear classification head. We proceed in the same way as for the other experiments, calculating subgroup latent distances on the embeddings obtained on the CheXagent model (before any tuning), and then fine-tuning the model on different allocations of MIMIC data. 

We find trends that are highly consistent with the other foundation model experiments and with our hypothesis. For example, Figure \ref{fig:chexagent peft gains} shows that, using our hypothesis, one could decide to prioritise balancing with respect to X-ray view rather than sex, leading to a 3\% accuracy gain compared to 0\%. When looking at all the groups together, we also see a strong correlation of 0.78 between latent subgroup TV distance and sensitivity to subgroup allocation (i.e. slope of loss change across allocations), as shown in this Figure \ref{fig:chexagent peft loss slope}. These results further expand the scope of our results and evidence for our hypothesis.

\begin{figure}[h!]
    \centering
    
    \includegraphics[width=0.5\linewidth]{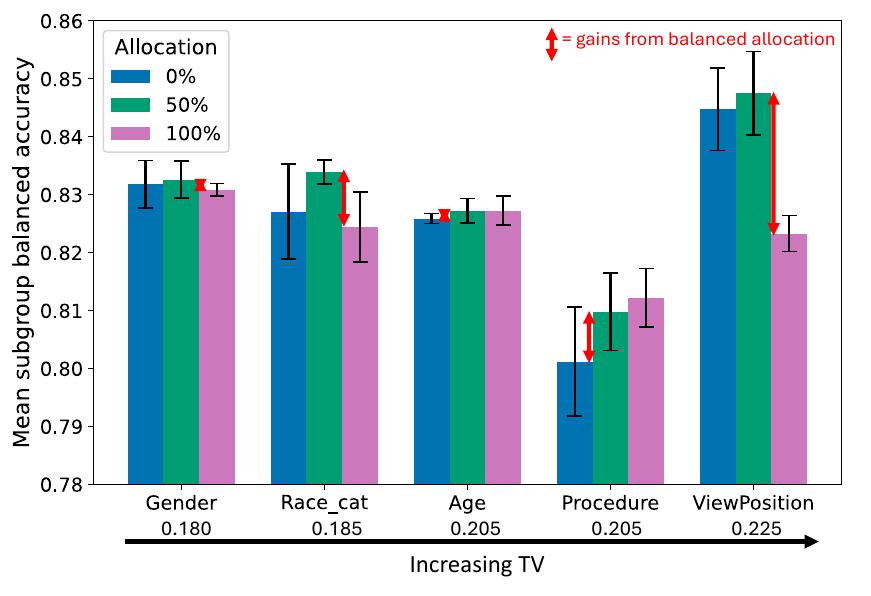}
    \caption{In CheXagent foundation model parameter-efficient fine-tuning, selecting a balanced allocation for imaging subgroups increases subgroup accuracy by about 0.03, while it has less importance for demographic subgroups, as predicted by their reduced total variation distance. We report mean accuracy and standard deviations for 3 fine-tuning runs, with red arrows for potential gains from balanced allocations.}
    \label{fig:chexagent peft gains}
\end{figure}
\begin{figure}[h!]
    \centering
    
    \includegraphics[width=1\linewidth]{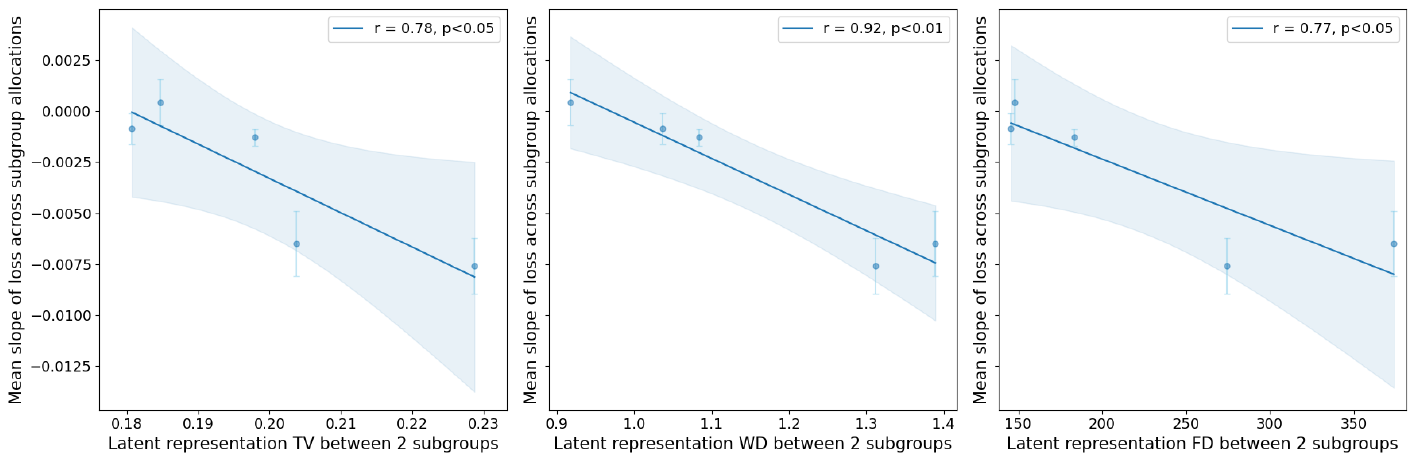}
    \caption{Correlation between latent representation separation and subgroup allocation sensitivity in CheXagent (vision-language model) parameter-efficient fine-tuning. Results confirm that our hypothesis generalises beyond task-specific models. Each dot represents mean distance and loss slope for one subgroup, averaged across 9 fine-tuning runs, with bars corresponding to standard deviations, and Pearson correlation coefficients also shown.}
    \label{fig:chexagent peft loss slope}
\end{figure}


\clearpage
\section{Ablations on fine-tuning budget $K$}
\label{sec: appendix-finetuning-budget}
We repeat the MIMIC experiments with a smaller fine-tuning budget $K$ to explore whether changing sample size modifies our results. We find that smaller budgets show slightly higher absolute magnitudes of allocation sensitivity. This is most likely due to larger sample sizes allowing the model to learn more robust disease representations which generalise better across allocations. However, correlations are equally significant across fine-tuning budgets, suggesting our latent representation hypothesis is valid across dataset sizes.

\begin{figure}[h!]
    \centering
    \includegraphics[width=1\linewidth]{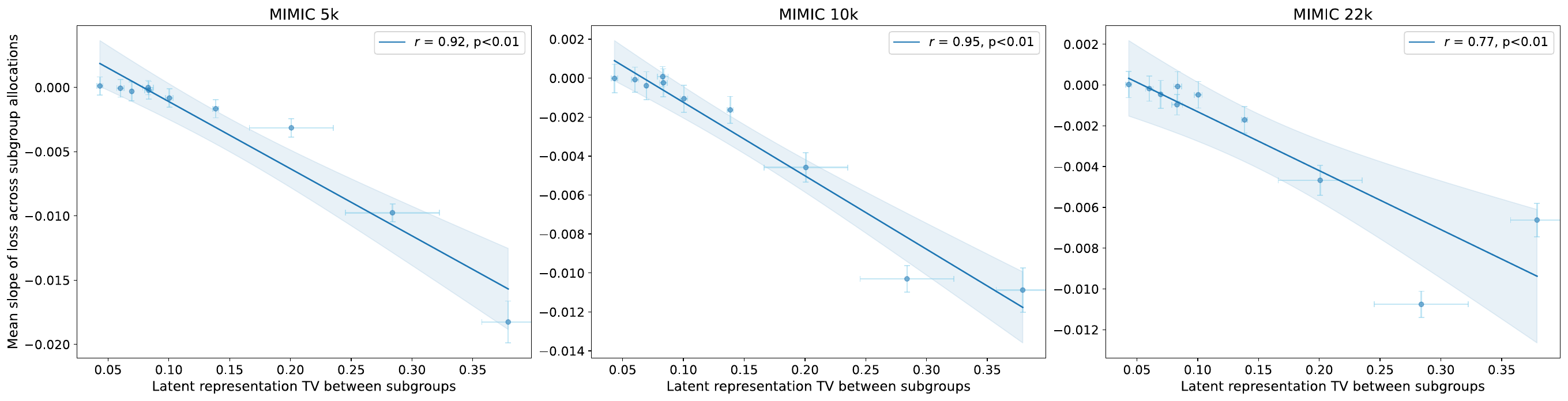}
    \caption{Correlation between representation separation and subgroup allocation sensitivity is strong across fine-tuning dataset sizes in MIMIC. Each dot represents mean distance and loss slope for one subgroup, averaged across 9 fine-tuning runs, with bars corresponding to standard deviations, and Pearson correlation coefficients also shown.}
    \label{fig:mimic tv corr acros sizes}
\end{figure}

\clearpage
\section{Implementation of TV-based Latent Regularisation}
\label{sec: appendix-operationalisation}

During pre-training, we augment the standard cross-entropy loss with a
differentiable surrogate for the TV distance between
the embedding distributions of the two view-position groups. For each
mini-batch, we compute a Mahalanobis-squared distance between the
group-wise mean last-layer embeddings within each label (i.e.,
conditioned on $Y$), and average these distances. This
quantity acts as a differentiable surrogate for the TV divergence and is scaled by
a regularisation hyperparameter $\lambda$. The resulting loss is
\[
\mathcal{L}
= \mathcal{L}_{\text{task}}
+ \lambda \, \widehat{\mathrm{TV}}(Z \mid A, Y),
\]
where $A$ denotes the view-position group and $Z$ the last-layer
embeddings. No gradient is taken through density estimates; the
surrogate is computed directly from batch statistics and therefore adds
minimal computational overhead. In practice, we use $\lambda = 0.01$, as this gives approximately equal weighting to the CE objective and the distance estimate (which is on the order of 100). To evaluate sensitivity to subgroup allocation, we fine-tune this model at with different allocations as in our standard pipeline.

Of course, this distance metric is not equivalent to TV, notably because the Mahalanobis distance is unbounded and because there can be cases where the Mahalanobis distance is 0 but TV is still high. However, this proxy serves all our required purposes, as we observe that using it during training reduces TV by 13\%, hence allowing us to test whether reducing TV decreases sensitivity to subgroup allocation.

We conduct these experiments in MIMIC, for five subgroups which initially showed a high TV distance. This includes groups based on age, X-ray view, procedure, presence of support devices, and orientation. These results show that regularisation completely removes any sensitivity to subgroup allocation (Table \ref{tab:regularisation exps}). This does come at the cost of reduced overall accuracy, although this unsurprising given that our aim was not to create the most optimal invariant representation learning model, it was simply to test whether our hypothesis is stronger than a simple correlation. These preliminary results suggest that this is indeed the case, as reducing latent separation actually leads to reduced sensitivity to subgroup allocation. 

\begin{table}[h]
\centering
\caption{Full results from regularisation experiments, including mean total variation distance (TV) between MIMIC subgroups, overall accuracy, and subgroup accuracy change across allocations.}
\begin{tabular}{lccc}
\toprule
\textbf{Method} & \textbf{TV} (mean $\pm$ std) & \textbf{Overall acc.} (mean $\pm$ std) & \textbf{Subgroup acc.} change (mean $\pm$ std) \\
\midrule
\textbf{Baseline}    & $0.995 \pm 0.006$ & $0.794 \pm 0.004$ & $0.010 \pm 0.002$ \\
\textbf{Regularised} & $0.87 \pm 0.012$ & $0.737 \pm 0.006$ & $-0.000 \pm 0.002$ \\
\bottomrule
\end{tabular}
\label{tab:regularisation exps}
\end{table}



\end{document}